\documentclass{article}
\usepackage{times}
\usepackage{epsf}
\usepackage{soul}
\usepackage{url}
\usepackage[hidelinks]{hyperref}
\usepackage[utf8]{inputenc}
\usepackage[small]{caption}
\usepackage{graphicx}
\usepackage{amsmath}
\usepackage{amsthm}
\usepackage{booktabs}
\usepackage{algorithm}
\usepackage{algorithmicx}
\usepackage{version}
\usepackage{amssymb,amsmath,amsfonts,graphicx,algpseudocode} 

\newcommand{\lpr}{{\underline{P}}}
\newcommand{\upr}{{\overline{P}}}
\newcommand{\cred}{{\mathcal{P}}}
\newtheorem{exam}{Example}
\newtheorem{proposition}{Proposition}
\newtheorem{definition}{Definition}
\newtheorem{lemma}{Lemma}
\newtheorem{corollary}{Corollary}
\newtheorem{example}{Example}
\newtheorem{theorem}{Theorem}

\newcommand{\ocap}{\textcircled{$\scriptstyle{\cap}$}}

\newcommand{\ocup}{\textcircled{$\scriptstyle{\cup}$}}

\title{The intersection probability: betting with probability intervals}

\author{
    Fabio Cuzzolin
    \\
    Visual Artificial Intelligence Laboratory\\ Oxford Brookes University, Oxford, UK
    \\
    fabio.cuzzolin@brookes.ac.uk
}

\begin{document}

\maketitle

\begin{abstract}
Probability intervals are an attractive tool for reasoning under uncertainty. Unlike belief functions, though, they lack a natural probability transformation to be used for decision making in a utility theory framework. In this paper we propose the use of the intersection probability, a transform derived originally for belief functions in the framework of the geometric approach to uncertainty, as the most natural such transformation. We recall its rationale and definition, compare it with other candidate representives of systems of probability intervals, discuss its credal rationale as focus of a pair of simplices in the probability simplex, and outline a possible decision making framework for probability intervals, analogous to the Transferable Belief Model for belief functions.
\end{abstract}

\section{Introduction}

An estimation or a decision problem $Q$ usually involvesknowing in which state we are, where the possible states of the world are often assumed to belong to a finite set $\Theta = \{ x_1,...,x_n \}$. Our uncertainty about the outcome of $Q$ can be described in many ways: the classical option is to assume a probability distribution on $\Theta$.
This, however, can only model the \emph{aleatory} uncertainty about the problem, in which the outcome is random, but the probability distribution that governs the process is fully known.
For instance, if a person plays a fair roulette wheel they will not, by any means, know the outcome in advance, but they will nevertheless be able to predict the long-term frequency with which each outcome manifests itself (1/36).

In opposition, in practical situations we are required to incorporate imprecise measurements and people's opinions in our knowledge state, or need to cope with missing or scarce information. A more cautious approach is therefore to assume that we have no access to the `correct' probability distribution, but that the available evidence provides us with some constraint on this unknown distribution.
This more fundamental type of uncertainty is referred to as \emph{epistemic} \cite{hullermeier2021aleatoric}, and is caused by lack of knowledge about the very process that generates the data. Suppose that the player is presented with ten different doors, which lead to rooms each containing a roulette wheel modelled by a different probability distribution. They will then be uncertain about the very game they are supposed to play. How will this affect their betting behaviour, for instance?

Numerous mathematical theories of epistemic uncertainty have been proposed, starting from de Finetti's pioneering work on subjective probability \cite{DeFinetti74}. 
To list just the most impactful efforts, we could mention possibility theory \cite{Zadeh78,Dubois90}, credal sets \cite{levi80book,kyburg87bayesian}, monotone capacities \cite{wang97choquet}, random sets \cite{Nguyen78} and imprecise probability theory \cite{walley91book}. New original foundations of subjective probability in behavioural terms \cite{walley00towards} or by means of game theory \cite{shafer01book} have been put forward.

One of the simplest such approaches is given by \emph{probability intervals} \cite{tessem92interval,decampos94,halpern03book}: the probability values $p(x)$ of the elements of the decision space $\Theta$ are assumed to belong to an interval $l(x) \leq p(x) \leq u(x)$ delimited by a lower bound $l(x)$ and an upper bound $u(x)$. This encodes a possible (convex) set of probabilities, usually called \emph{credal set} \cite{levi80book}, from which decision can then be taken. When considering an associated decision problem $Q$, many extensions of the classical expected utility rule proposes to extract not one, but multiple potentially optimal decisions~\cite{Troffaes07}. 

There are many situations, however, in which one must converge to a unique decision. 
Decision rules producing one optimal decision in the imprecise-probabilistic literature either focus on specific bounds (e.g., maxi-min \cite{Wald1945,Troffaes07} or maxi-max rules), not accounting for the whole representation, or require solving complex optimisation problems (e.g., selecting the maximal entropy distribution).

An alternative approach which has been often studied within the belief function theory \cite{cobb03on} consists in approximating a complex uncertainty measure (such as a belief function) by a single probability distribution, from which a unique optimal decision can be deduced. 
This is, for instance,  the case of the Transferable Belief Model \cite{smets88beliefversus}, in which the so called \emph{pignistic transformation} \cite{smets05ijar} is employed for this purpose. 

Similarly to the case of belief functions, it could be useful to apply such a transformation to reduce a set of probability intervals to a single probability distribution prior to actually making a decision. However, this problem has been quite neglected so far. One could of course pick a representative from the corresponding credal set, but it makes sense to wonder whether a transformation inherently designed for probability intervals as such could be found.

In this paper we argue that the natural candidate for the role of a probability transform of probability intervals is the \emph{intersection probability}, originally identified in the context of the geometric approach to uncertainty \cite{cuzzolin07smcb}. Even though originally introduced as a probability transform for belief functions \cite{cuzzolin07smcb}, the intersection probability is (as we show here) inherently associated with probability intervals, in which context its rationale clearly emerges.
\\
It can be shown that the intersection probability is the only probability distribution that behave homogeneously in each element $x$ of the frame $\Theta$, i.e., it assigns the same fraction of the available probability interval to each element of the decision space.

\subsection{Contributions and paper outline} \label{sec:paper-outline}


We first recall the basic notions of the theories of lower and upper probabilities, interval probabilities and belief functions (Section \ref{sec:theories}), to then quickly review the prior art on probability transform and its geometric interpretation (Section \ref{sec:transform}).

We then formally define the intersection probability and its rationale (Section \ref{sec:definition}),
showing that it can be defined for any interval probability system as the unique probability distribution obtained by assigning the same fraction of the uncertainty interval to all the elements of the domain. We compare it with other possible representatives of interval probability systems, 
and recall its geometric interpretation in the space of belief functions and the justification for its name that derives from it (Section \ref{sec:geometric}).

In Section \ref{sec:credal} we extensively illustrate the credal rationale for the intersection probability as focus of the pair of lower and upper simplices associated with the interval probability system.

As a belief function determines itself an interval probability system, the intersection probability exists for belief functions too and can therefore be compared with classical approximations of belief functions like the pignistic function \cite{smets88beliefversus} and relative plausibility and belief \cite{cuzzolin2008dual,cuzzolin2008semantics} of singletons, or more recent approximations proposed by Sudano \cite{sudano01icif}.
In Section \ref{sec:relations} we thus analyse the relations of intersection probability with other probability transforms of belief functions, while in Section \ref{sec:operators} we discuss its properties with respect to affine combination and convex closure.

The potential use of the intersection probability to bet on interval probability systems, in a framework analogous to the Transferable Belief Model, is outlined in Section \ref{sec:decision}, which concludes the paper.


\section{Uncertainty theories} \label{sec:theories}

\subsection{Lower and upper probabilities} \label{sec:lower-upper}

A \emph{lower probability} $\lpr$ is a function from $2^\Theta$, the power set of $\Theta$, to the unit interval $[0,1]$. A lower probability is associated with a dual \emph{upper probability} $\upr$, defined for any $A \subseteq \Theta$ as $\overline{P}(A) = 1 - \underline{P}(A^c)$, where $A^c$ is the complement of $A$ (this means that we can simply focus on one of the two measures, provided it is defined for all subsets). A lower probability $\lpr$ can also be associated with a (closed convex) set 
\begin{equation}\label{eq:credal}
\cred(\lpr)= \Big \{p: P(A) \geq \lpr(A), \forall A \subseteq \Theta \Big \}
\end{equation}
of probability distributions $p$ whose measure $P$ dominates $\lpr$. Such a polytope or convex set of probability distributions is usually called a \emph{credal set}. A lower probability $\lpr$ will be called \emph{consistent} if $\cred(\lpr)$ and \emph{tight} if $\inf_{p \in \cred(\lpr)} P(A) = \lpr(A)$ (respectively it `avoids sure loss' and is `coherent' in Peter Walley's~\cite{walley91book} terms). Consistency means that lower bound constraints $\lpr(A)$ can be satisfied, while tightness means that $\lpr$ is the lower envelope on subsets of $\cred(\lpr)$. Note that not all convex sets of probabilities can be described by only focusing on events (see Walley~\cite{walley91book}), however they will be sufficient here. The notation $\mathcal{P}$ 
will simply denote the set of all possible probabilities on $\Theta$.

\subsection{Probability intervals} \label{sec:probability-intervals}

Dealing with general lower probabilities defined on $2^\Theta$ can be difficult when $\Theta$ is big, and it may be interesting in applications to focus on simpler models. One popular and practical model used to model such kind of uncertainty are probability intervals. 

A \emph{set of probability intervals} or \emph{interval probability system} is a system of constraints on the probability values of a probability distribution $p:\Theta \rightarrow [0,1]$ on a finite domain $\Theta$ of the form
\begin{equation} \label{eq:credal-interval}
\mathcal{P}(l,u) \doteq \Big \{p: l(x) \leq p(x) \leq u(x), \forall x \in \Theta \Big \}.
\end{equation}
Probability intervals have been introduced as a tool for uncertain reasoning in \cite{decampos94}, where combination and marginalization of intervals were studied in detail. In \cite{decampos94} the authors also studied the specific constraints for such intervals to be consistent and tight. 

As pointed out for instance in \cite{unclog08book}, a typical way in which probability intervals arise is through measurement errors. As a matter of fact, measurements can be inherently of interval nature (due to the finite resolution of the instruments). In that case the \emph{probability} interval of interest is the class of probability measures consistent with the \emph{measured} interval. A set of constraints of the form (\ref{eq:credal-interval}) determines a credal set, which is just a sub-class of sets described by lower and upper probabilities. 

The lower and upper probabilities induced by $\mathcal{P}(l,u)$  on any subset $A \subseteq \Theta$ from bounds $(l,u)$ can be obtained using the simple formulas:
\begin{equation}\label{eq:lowup-from-int}
\lpr(A)=\max \left \{ \sum_{x \in A} l(x), 1- \sum_{x \not\in A} u(x) \right \}, \; \lpr(A)=\min \left \{ \sum_{x \in A} u(x), 1- \sum_{x \not\in A} l(x) \right \}.
\end{equation}

\subsection{Belief functions} \label{sec:belief-functions}

A special class of lower and upper probabilities is provided by belief and plausibility measures.
Namely, 
a \emph{basic probability assignment} (BPA) \cite{Shafer76} is a set function \cite{denneberg99interaction,dubois86logical} $m : 2^\Theta\rightarrow[0,1]$ such that 
\[
m(\emptyset)=0, \quad
\sum_{A\subset\Theta} m(A)=1.
\]
Subsets of $\Theta$ whose mass values are non-zero are called \emph{focal elements} of $m$.
The \emph{belief function} (BF) associated with a BPA $m : 2^\Theta\rightarrow[0,1]$ is the set function $Bel : 2^\Theta\rightarrow[0,1]$ defined as
\begin{equation} \label{eq:belief-function}
Bel(A) = \sum_{B\subseteq A} m(B). 
\end{equation}
The corresponding \emph{plausibility function} is 
\[
Pl(A) \doteq \sum_{B\cap A\neq \emptyset} m(B) \geq Bel(A).
\]
Note that belief functions can also be equivalently defined in axiomatic terms \cite{Shafer76}.
 
Classical probability measures on $\Theta$ are a special case of belief functions (those assigning mass to singletons only), termed \emph{Bayesian belief functions}.
A BF is said to be \emph{consonant} if its focal elements $A_1,...,A_m$ are nested: $A_1 \subset \cdots \subset A_m$, and corresponds to a possibility measure \cite{dubois88possibility,dubois2012possibility,Shafer76,cuzzolin2021springer}.

\subsubsection{Combination} 

In belief theory conditioning is replaced by the notion of (associative) \emph{combination} of any number of belief functions. 

The \emph{Dempster combination} $Bel_1 \oplus Bel_2$ 
of two belief functions 
on $\Theta$ is the unique BF there with as focal elements all the {non-empty} intersections of focal elements of $Bel_1$ and $Bel_2$, and basic probability assignment
\begin{equation} \label{eq:dempster}
m_{\oplus}(A) = \frac{m_\cap(A)} {1- m_\cap(\emptyset)},
\end{equation}
where 
\begin{equation} \label{eq:intersection}
m_\cap(A) = \sum_{B \cap C = A} m_1(B) m_2(C) 
\end{equation}
and $m_i$ is the BPA of the input belief function $Bel_i$.

Nevertheless, Dempster's combination naturally induces a conditioning operator. Given a conditioning event $A \subset \Theta$, the `logical' or \emph{categorical} belief function $Bel_A$ such that $m(A)=1$ is combined via Dempster's rule with the a-priori belief function $Bel$. The resulting BF $Bel \oplus Bel_A$ is the {conditional belief function given $A$} \emph{a la Dempster}, denoted by $Bel_\oplus(A|B)$.

Many alternative combination rules have since been defined \cite{Klawonn:1992:DBT:2074540.2074558,yager87on,dubois88representation,DENOEUX2008234}, often associated with a distinct approach to conditioning \cite{Denneberg1994,fagin91new,suppes1977}. An exhaustive review of these proposals can be found in \cite{cuzzolin2021springer}, Section 4.3.

Rather than normalising (as in Dempster's rule) or reassigning the conflicting mass $m_\cap(\emptyset)$ to other non-empty subsets, 
Philippe Smets's \emph{conjunctive rule} leaves the conflicting mass with the empty set,
\begin{equation} \label{eq:conjunctive}
m_{\text{\text{\ocap}}}(A) = \left \{ \begin{array}{ll} m_\cap (A) & \emptyset \neq A \subseteq \Theta, 
\\ 
m_\cap(\emptyset) & A = \emptyset, \end{array} \right .
\end{equation}
and thus is applicable to \emph{unnormalised} belief functions \cite{ubf}.

In Dempster's original random-set idea, consensus between two sources is expressed by the intersection of the supported events (\ref{eq:intersection}). When the \emph{union} of the supported propositions is taken to represent such a consensus instead, we obtain what Smets called the \emph{disjunctive} rule of combination,
\begin{equation} \label{eq:disjunctive}
m_{\text{\text{\ocup}}}(A) = \sum_{B \cup C = A} m_1(B) m_2(C).
\end{equation}
It is interesting to note that under disjunctive combination, 
\[
Bel_1 \ocup Bel_2 (A) = Bel_1(A) \ast Bel_2(A), 
\]
i.e., the belief values of the input belief functions are simply multiplied.

\subsubsection{Belief functions and other measures}

Each belief function $Bel$ uniquely identifies a credal set \cite{kyburg87bayesian}
\begin{equation} \label{eq:consistent}
\mathcal{P}[Bel] = \{ P \in \mathcal{P} : P(A) \geq Bel(A) \}
\end{equation}
(where $\mathcal{P}$ is the set of all probabilities one can define on $\Theta$), of which it is its lower envelope: $Bel(A) = \underline{P}(A)$.  Belief functions are thus a special case of lower probabilities (Section \ref{sec:lower-upper}).
The corresponding plausibility measure is the upper probability of an event $A$: $Pl(A) = \overline{P}(A)$.
The probability intervals resulting from Dempster's updating of the credal set associated with a BF, however, are included in those resulting from Bayesian updating \cite{kyburg87bayesian}.

Belief functions are also infinitely \emph{monotone capacities} \cite{Choquet53,sugeno74fuzzy}, and a special case of \emph{coherent lower previsions} \cite{walley91book,walley00towards}. Finally, every belief function specifies a unique \emph{probability box} \cite{Ferson03pboxes}, i.e., a class of cumulative distribution functions delimited by two lower and upper bounds.

\subsection{The geometry of uncertainty measures}

\emph{Geometry} has been proposed by this author and others as a unifying language for the field \cite{cuzzolin18belief-maxent,Cuzzolin99,cuzzolin05isipta,cuzzolin00mtns,cuzzolin13fusion,gennari02-integrating,gong2017belief,black97geometric,rota97book,ha98geometric,wang91geometrical}, possibly in conjunction with an algebraic view \cite{cuzzolin00rss,cuzzolin01bcc,cuzzolin08isaim-matroid,cuzzolin01lattice,cuzzolin05amai,cuzzolin07bcc,cuzzolin14algebraic}.

Indeed, uncertainty measures can be seen as points of a suitably complex geometric space, and there manipulated (e.g. combined, conditioned and so on) \cite{cuzzolin01thesis,cuzzolin2008geometric,cuzzolin2021springer}.
Much work has been focusing on the geometry of belief functions, which live in a convex space termed the \emph{belief space}, which can be described both in terms of a simplex (a higher-dimensional triangle) and in terms of a recursive bundle structure \cite{cuzzolin01space,cuzzolin03isipta,cuzzolin14annals,cuzzolin14lap}. The analysis can be extended to Dempster's rule of combination by introducing the notion of a conditional subspace and outlining a geometric construction for Dempster's sum \cite{cuzzolin02fsdk,cuzzolin04smcb}.
The combinatorial properties of plausibility and commonality functions, as equivalent representations of the evidence carried by a belief function, have also been studied \cite{cuzzolin08pricai-moebius,cuzzolin10ida}. The corresponding spaces are simplices which are congruent to the belief space.
\\
Subsequent work extended the geometric approach to other uncertainty measures, focusing in particular on possibility measures (consonant belief functions) \cite{cuzzolin10fss} and consistent belief functions \cite{cuzzolin11-consistent,cuzzolin09isipta-consistent,cuzzolin08isaim-simplicial}, in terms of simplicial complexes \cite{cuzzolin04ipmu}. Analyses of belief functions in terms credal sets have also been conducted \cite{cuzzolin08-credal,antonucci10-credal,burger10brest}.

The geometry of the relationship between measures of different kinds has also been extensively studied \cite{cuzzolin05hawaii,cuzzolin09-intersection,cuzzolin07ecsqaru,cuzzolin2010credal}, with particular attention to the problem of transforming a belief function into a classical probability measure \cite{Cobb03isf,voorbraak89efficient,Smets:1990:CPP:647232.719592} (see Section \ref{sec:transform}). One can distinguish between an \emph{affine} family of probability transformations \cite{cuzzolin07smcb} (those which commute with affine combination in the belief space), and an \emph{epistemic} family of transforms \cite{cuzzolin07report}, formed by the relative belief and relative plausibility of singletons \cite{cuzzolin08unclog-semantics,cuzzolin2008semantics,CUZZOLIN2012786,cuzzolin06-geometry,cuzzolin10amai}, which possess dual properties with respect to Dempster's sum \cite{cuzzolin2008dual}.
The problem of finding the possibility measure which best approximates a given belief function \cite{aregui08constructing} can also be approached in geometric terms \cite{cuzzolin09ecsqaru,cuzzolin11isipta-consonant,cuzzolin14lp,Cuzzolin2014tfs}. In particular, approximations induced by classical Minkowski norms can be derived and compared with classical outer consonant approximations \cite{Dubois90}.
Minkowski consistent approximations of belief functions in both the mass and the belief space representations can also be derived \cite{cuzzolin11-consistent}.

The geometric approach to uncertainty can also be applied to the conditioning problem \cite{lehrer05updating}. Conditional belief functions can be defined as those which minimise an appropriate distance between the original belief function and the `conditioning simplex' associated with the conditioning event \cite{cuzzolin10brest,cuzzolin11isipta-conditional}. 

Recent papers on this topic include \cite{luo2020vector,pan2020probability,long2021visualization}.

\section{Probability transform} \label{sec:transform}

\subsection{Probability transforms of belief functions} \label{sec: transform-belief}

The relation between belief and probability in the theory of evidence has been and continues to be an important subject of study\cite{Weiler2003approximation,Kramosil95,bauer97approximation,Yaghlane01ecsqaru,Denoeux01ijufk,Denoeux02ijar,HAENNI2002103}. 
A probability transform mapping belief functions to probability measures can be instrumental in addressing a number of issues: mitigating the inherently exponential complexity of belief calculus \cite{bauer97approximation}, making decisions via the probability distributions obtained in a utility theory framework \cite{Smets:1990:CPP:647232.719592} and obtaining pointwise estimates of quantities of interest from belief functions (e.g., the pose of an articulated object in computer vision: see \cite{cuzzolin14lap}, Chapter 8, or \cite{cuzzolin05isipta}).

As both belief and probability measures can be assimilated into points of a Cartesian space \cite{cuzzolin2021springer}, the problem can (as mentioned) be posed in a geometric setting. 
Without loss of generality, we can define a \emph{probability transform} as a mapping from the space of belief functions $\mathcal{B}$ on the domain of interest to the probability simplex $\mathcal{P}$ there,
\[
\begin{array}{lllll}
\mathcal{PT} & : & \mathcal{B} & \rightarrow & \mathcal{P}, 
\\ 
& & Bel \in \mathcal{B} & \mapsto & \mathcal{PT}[Bel] \in \mathcal{P}, 
\end{array}
\]
such that an appropriate distance function or similarity measure $d$ from $Bel$ is minimised \cite{daniel06ijis}:
\begin{equation} \label{eq:pt}
\mathcal{PT}[Bel] = \arg \min_{P \in \mathcal{P}} d(Bel,P).
\end{equation}
A minimal, sensible requirement is for the probability which results from the transform to be compatible with the upper and lower bounds that the original belief function $Bel$ enforces {on the singletons only}, rather than on all the focal sets. Thus, this does not require probability transforms to adhere to the upper--lower probability semantics of belief functions.
As a matter of fact, some important transforms of this kind are not compatible with such semantics.

Many such transformations have been proposed, according to different criteria \cite{Lowrance90,tessem93approximations,Yaghlane01ecsqaru,Denoeux01ijufk,Denoeux02ijar,bauer97approximation}.
In Smets's transferable belief model \cite{smets88beliefversus,smets94transferable}, in particular, decisions are made by resorting to the \emph{pignistic probability}:
\begin{equation} \label{eq:pignistic}
BetP[Bel](x) = \sum_{A\supseteq \{x\}} \frac{m(A)}{|A|},
\end{equation}
which is
the output of the \emph{pignistic transform}.

An interesting approach to the problem seeks approximations which enjoy commutativity properties with respect to a specific combination rule, in particular Dempster's sum \cite{Dempster68a,Dempster68b}. This is the case of the \emph{relative plausibility of singletons} \cite{voorbraak89efficient}, the unique probability that, given a belief function $Bel$ with plausibility $Pl(A) = 1 - Bel(A^c)$, assigns to each singleton its normalized plausibility:\footnote{With a harmless abuse of notation, we will often denote the values of belief functions and plausibility functions on a singleton $x$ by $m(x), Pl(x)$ rather than by $m(\{x\}),Pl(\{x\})$.} 
\begin{equation}\label{eq:relplaus}
\tilde{Pl}[Bel](x) = \frac{Pl(x)}{\sum_{y\in\Theta} Pl(y)}.
\end{equation}
Its properties have been later analyzed by Cobb and Shenoy \cite{Cobb03isf,cobb06ijar}.
Voorbraak proved that his (in our terminology) relative plausibility of singletons $\tilde{Pl}[Bel]$ is a perfect representative of $Bel$ when combined with other probabilities $P \in \mathcal{P}$ through Dempster's rule $\oplus$:
\begin{equation}\label{eq:voorbraak}
\tilde{Pl}[Bel] \oplus P = Bel \oplus P \quad \forall P \in\mathcal{P}.
\end{equation}
Dually, a \emph{relative belief transform} $\tilde{Bel} : \mathcal{B} \rightarrow \mathcal{P}$, $Bel \mapsto \tilde{Bel}[Bel]$ mapping each belief function $Bel$ to the corresponding \emph{relative belief of singletons} \cite{cuzzolin2008dual,cuzzolin08unclog-semantics,haenni08aggregating,daniel06ijis},
\begin{equation}\label{eq:btilde}
\tilde{Bel}[Bel](x) = \frac{Bel(x)}{\sum_{y \in \Theta} Bel(y)},
\end{equation}
can be defined. The notion of a relative belief transform (under the name of `normalised belief of singletons') was first proposed by Daniel in \cite{daniel06ijis}. Some analyses of the relative belief transform and its close relationship with the (relative) plausibility transform were presented in \cite{cuzzolin2008dual,cuzzolin08unclog-semantics}.

\subsection{Geometric approaches} \label{sec: transform-geometric}

Only a few authors have in the past posed the study of the connections between belief functions and probabilities in a geometric setting.
In particular, Ha and Haddawy \cite{Ha} proposed an `affine operator', which can be considered a generalisation of both belief functions and interval probabilities, and can be used as a tool for constructing convex sets of probability distributions. In their work, uncertainty is modelled as sets of probabilities represented as `affine trees', while actions (modifications of the uncertain state) are defined as tree manipulators.
In a later publication \cite{ha98geometric}, the same authors presented an interval generalisation of the probability cross-product operator, called the `convex-closure' (cc) operator, 
analysed the properties of the cc operator relative to manipulations of sets of probabilities and presented interval versions of Bayesian propagation algorithms based on it. Probability intervals were represented there in a computationally efficient fashion by means of a data structure called a `pcc-tree', in which branches are annotated with intervals, and nodes with convex sets of probabilities.

The intersection probability introduced in this paper is somewhat related to Ha's cc operator, as it commutes (at least under certain conditions) with affine combination, and is therefore part of the \emph{affine family} of Bayesian transforms of which Smets's pignistic transform \cite{smets2005decision} is the foremost representative.

\section{The intersection probability} \label{sec:definition}

When our uncertainty is described by imprecise probabilities, it may be desirable for some reasons to transform this knowledge into a classical unique probability. Such reasons include the need to take a unique optimal decision, the need to use classical probabilistic calculus (e.g., for efficiency), or more simply the will to obtain a unique probability from partial probabilistic information. Existing proposals are general in scope but rather complex, e.g., they imply solving a convex optimization problem.  

\subsection{Definition} \label{sec:definition-intersection-probability}

There are clearly many ways of selecting a single measure to represent a collection of probability intervals (\ref{eq:credal-interval}). Note, however, that each of the intervals $[l(x),u(x)]$, $x \in\Theta$, carries the same weight within the system of constraints (\ref{eq:credal-interval}), as there is no reason for the different elements $x$ of the domain to be treated differently.
It is then sensible to require that the desired representative probability should behave homogeneously in each element $x$ of the frame $\Theta$. 

Mathematically, this translates into seeking a probability distribution $p : \Theta \rightarrow [0,1]$ such that
\[
p(x) = l(x) + \alpha (u(x) - l(x))
\]
for all the elements $x$ of $\Theta$, and some constant value $\alpha \in [0,1]$ (see Fig. \ref{fig:notion-intersection}). This value needs to be between 0 and 1 in order for the sought probability distribution $p$ to belong to the interval.
\begin{figure}[ht!]
\begin{center}
\includegraphics[width = \textwidth]{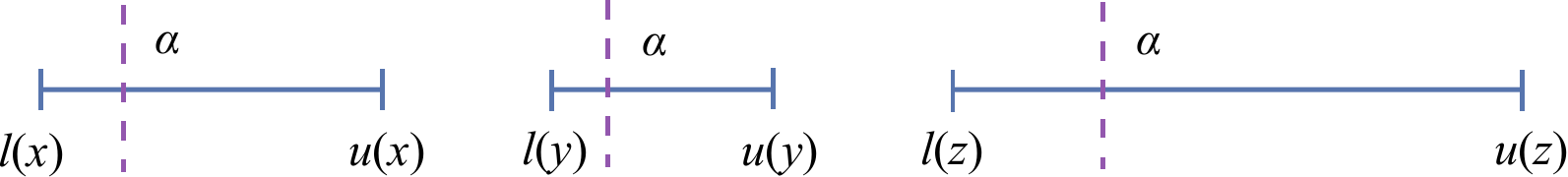}
\end{center}
\caption{An illustration of the notion of the intersection probability for an interval probability system $(l,u)$ on $\Theta = \{x,y,z\}$ (\ref{eq:credal-interval}). \label{fig:notion-intersection}}
\end{figure}
It is easy to see that there is indeed a \emph{unique} solution to this problem. It suffices to enforce the normalisation constraint
\[
\sum_x p(x) = \sum_x \Big [ l(x) + \alpha (u(x) - l(x)) \Big ] = 1
\]
to understand that the unique value of $\alpha$ is given by
\begin{equation} \label{eq:beta-interval}
\alpha = \beta[(l,u)] \doteq \frac{1 - \sum_{x \in \Theta} l(x)}{\sum_{x \in \Theta} \big(u(x) - l(x) \big)}.
\end{equation} 
\begin{definition}
The \emph{intersection probability} $p[(l,u)] : \Theta \rightarrow [0,1]$ associated with the interval probability system (\ref{eq:credal-interval}) is the probability distribution
\begin{equation} \label{eq:form1}
p[(l,u)](x) = \beta[(l,u)] u(x) + (1 - \beta[(l,u)]) l(x),
\end{equation}
with $\beta[(l,u)]$ given by (\ref{eq:beta-interval}). 
\end{definition}
The ratio $\beta[(l,u)]$ \eqref{eq:beta-interval} measures the fraction of each interval $[l(x),u(x)]$ which we need to add to the lower bound $l(x)$ to obtain a valid probability function (adding up to one).

It is easy to see that when $(l,u)$ are a pair of belief/plausibility measures $(Bel,Pl)$, we can define the intersection probability for belief functions as well. Although originally defined by geometric means \cite{cuzzolin07smcb}, the intersection probability is thus in fact `the' rational probability transform for general interval probability systems.

Note that  $p[(l,u)]$ can also be written as
\begin{equation} \label{eq:form2}
p[(l,u)](x) = l(x) + \left ( 1 - \sum_x l(x) \right ) R[(l,u)](x),
\end{equation}
where
\begin{equation} \label{eq:relative-uncertainty}
R[(l,u)](x) \doteq \frac{u(x) - l(x)}{\sum_{y\in\Theta} (u(y) - l(y))} = \frac{\Delta(x)}{\sum_{y\in\Theta} \Delta(y)}.
\end{equation}
Here $\Delta(x)$ measures the width of the probability interval for $x$, whereas  
$R[(l,u)] : \Theta \rightarrow [0,1]$ measures how much the uncertainty in the probability value of each singleton `weighs' on the total width of the interval system (\ref{eq:credal-interval}). We thus term it the \emph{relative uncertainty} of singletons. Therefore, we can say that $p[(l,u)]$ distributes the mass $(1 - \sum_x l(x))$ to each singleton $x \in \Theta$ according to the relative uncertainty $R[(l,u)](x)$ it carries for the given interval.

\begin{example} \label{exm:probint}

Consider as an example an interval probability system on a domain $\Theta = \{ x,y,z \}$ of size 3:
\begin{equation} \label{eq:interval-example}
\begin{array}{lll}
0.2 \leq p(x) \leq 0.8, \;\;\; & 0.4 \leq p(y) \leq 1, \;\;\; & 0.3 \leq p(z) \leq 0.3.
\end{array}
\end{equation}
Notice that there is no uncertainty at all on the value of $p(z) = 0.3$.
The widths of the corresponding intervals are $\Delta(x) = 0.6$, $\Delta(y) = 0.6$, $\Delta(z) = 0$ respectively.
The relative uncertainty on each singleton (\ref{eq:relative-uncertainty}) is therefore:
\begin{equation} \label{eq:erre-example}
\begin{array}{lll}
R[(l,u)](x) & = & \frac{\Delta(x)}{\sum_{w \in \Theta} \Delta(w)} = \frac{0.6}{1.2} = \frac{1}{2}, \\ R[(l,u)](y) & = & \frac{1}{2}, \\ R[(l,u)](z) & = & \frac{\Delta(z)}{\sum_{w \in\Theta } \Delta(w)} = \frac{0}{1.2} = 0.
\end{array}
\end{equation}
Computing the intersection probability is then really easy. By Equation (\ref{eq:beta-interval}) the fraction of the uncertainty $u(x) - l(x)$ on $p(x)$ we need to add to the lower bound $l(x)$ to get an admissible, normalized probability is
\[
\beta = \frac{1 - 0.2 - 0.4 - 0.3}{0.6 + 0.6} = \frac{0.1}{1.2} = \frac{1}{12}.
\]
The intersection probability (\ref{eq:form2}) has therefore values:
\[
\begin{array}{c}
p[(l,u)](x) = 0.2 + \frac{1}{12} 0.6 = 0.25, \hspace{5mm} p[(l,u)](y) = 0.4 + \frac{1}{12} 0.6 = 0.45, 
\\ \\
p[(l,u)](z) = 0.3 + \frac{1}{12} 0 = 0.3.
\end{array}
\]
Notice that the fact of having a zero-width interval for one of the singletons does not pose a problem for the intersection probability, which falls as expected inside the probability interval for all the elements of the domain.

According to its interpretation of Equation (\ref{eq:form2}), $p[(l,u)]$ is also the result of distributing the necessary mass $( 1 - \sum_x l(x) ) = 1 - 0.2 - 0.4 - 0.3 = 0.1$ to each singleton in proportion to the relative uncertainty $R[(l,u)]$ (Equation (\ref{eq:erre-example})) of their intervals:
\[
\begin{array}{c}
p[(l,u)](x) = 0.2 + 0.1 \frac{1}{2} = 0.25, \hspace{5mm} p[(l,u)](y) = 0.4 + 0.1 \frac{1}{2} = 0.45, \\ \\
p[(l,u)](z) = 0.3 + 0.1 \cdot 0 = 0.3.
\end{array}
\]
\end{example}

\subsection{Comparison with other interval representatives} \label{sec:comparison-interval}

It can be useful to briefly compare the proposed intersection probability with other possible representatives of an interval probability system (\ref{eq:credal-interval}).

\subsubsection{Comparison with the center of mass of $P[(l,u)]$}

The naive choice of picking the barycenter of each interval $[l(x),u(x)]$ to represent an interval probability system $(l,u)$, for instance, does not yield in general a valid probability function, for
\[
\sum_{x \in \Theta} \left [ l(x) + \frac{1}{2}(u(x) - l(x)) \right ] \neq 1.
\]
This marks the difference with the case of belief functions, for which the pignistic function has a strong interpretation as barycenter of the associated credal set.

\subsubsection{Comparison normalised lower and upper bounds}

For the probability interval system (\ref{eq:interval-belief}) determined by a belief function, 
\begin{equation} \label{eq:interval-belief}
(Bel, Pl) \doteq \big \{ p \in \mathcal{P} : Bel(x) \leq p(x) \leq Pl(x), \forall x \in \Theta \big \},
\end{equation}
the probabilities we obtain by normalizing lower $\tilde{l}(x) = l(x)/\sum_y l(y)$ or upper bound $\tilde{u}(x) = u(x)/\sum_y u(y)$ are \emph{not} guaranteed to be consistent with the interval itself.

For instance, if there exists an element $x \in \Theta$ such that $Bel(x) = Pl(x)$ (the interval has width zero for that element) we have that
\[
\tilde{Bel}(x) = \frac{m(x)}{\sum_y m(y)} > Pl(x), 
\quad 
\tilde{Pl}(x) = \frac{Pl(x)}{\sum_y Pl(y)} < Bel(x).
\]
Therefore, both relative belief and plausibility of singletons fall outside the interval system (\ref{eq:interval-belief}). 
This holds for a general collection of probability intervals (\ref{eq:credal-interval}), again marking the contrast with the behavior of the intersection probability.

\subsubsection{Comparison with Sudano's proposal}

In the belief functions framework, Sudano proposed in \cite{sudano03-equivalence} the following four probability transforms:
\begin{flalign} \label{eq:prpl}
\begin{array}{lll}
PrPl[Bel](x) & \doteq & \displaystyle \sum_{A \supseteq \{x\}} m(A) \frac{Pl(x)}{\sum_{y \in A} Pl(y)},
\end{array}
\end{flalign}
\begin{equation} \label{eq:prbel}
PrBel[Bel](x) \doteq \sum_{A \supseteq \{x\}} m(A) \frac{Bel(x)}{\sum_{y \in A} Bel(y)} = \sum_{A \supseteq \{x\}} m(A) \frac{m(x)}{\sum_{y \in A} m(y)},
\end{equation}
\begin{equation} \label{eq:prNpl}
PrNPl[Bel](x) \doteq \frac{1}{\Delta} \sum_{A \cap \{x\} \neq \emptyset} m(A) = \tilde{Pl}[Bel](x),
\end{equation}
\begin{equation} \label{eq:prapl}
PraPl[Bel](x) \doteq Bel(x) + \epsilon \cdot Pl(x), \;\;\; \epsilon = \frac{1- \sum_{y\in\Theta} Bel(y)}{\sum_{y \in \Theta} Pl(y)} = \frac{1 - k_{Bel}}{k_{Pl}},
\end{equation}
where
\begin{equation} \label{eq:kappas}
k_{Bel} \doteq \sum_{x \in \Theta} Bel(x), \quad k_{Pl} \doteq \sum_{x \in \Theta} Pl(x).
\end{equation}
The first two transformations are clearly inspired by the pignistic function (\ref{eq:pignistic}). While in the latter case the mass $m(A)$ of each focal element is redistributed homogeneously to all its elements $x \in A$, $PrPl[Bel]$ (\ref{eq:prpl}) redistributes $m(A)$ proportionally to the relative plausibility of a singleton $x$ \emph{inside $A$}. Similarly, $PrBel[Bel]$ (\ref{eq:prbel}) redistributes $m(A)$ proportionally to the relative \emph{belief} of a singleton $x$ within $A$.

The fourth transformation (\ref{eq:prapl}), $PraPl[Bel]$, is more related to the case of probability intervals and to the intersection probability. By Equation (\ref{eq:form1}),
\begin{equation} \label{eq:relation}
\begin{array}{lll}
p[Bel](x) & = & \displaystyle (1 - \beta[Bel]) \tilde{Bel}(x) k_{Bel} + \beta[Bel] \tilde{Pl}(x) k_{Pl},
\end{array}
\end{equation}
where
\[
(1 - \beta[Bel]) k_{Bel} + \beta[Bel] k_{Pl} = \frac{k_{Pl} - 1}{k_{Pl} - k_{Bel}} k_{Bel} + \frac{1 - k_{Bel}}{k_{Pl} - k_{Bel}} k_{Pl} = 1,
\]
i.e., $p[Bel]$ lies on the line joining the relative plausibility $\tilde{Pl}$ and the relative belief $\tilde{Bel}$ of singletons.
Here $\beta[Bel]$ is the value of (\ref{eq:beta-interval}) for a system of probability intervals associated with a belief function $Bel$, namely:
\begin{equation} \label{eq:beta}
\beta[Bel] \doteq \frac{1 - \sum_{x \in \Theta} Bel(x)}{\sum_{x \in \Theta} \big( Pl (x) - Bel (x) \big)}.
\end{equation} 

Just like the intersection probability (\ref{eq:relation}) and the relative uncertainty of singletons \cite{cuzzolin07smcb}, $PraPl[b]$ can also be expressed as an affine combination of relative belief and plausibility of singletons:
\begin{equation} \label{eq:prapl-interval}
PraPl[Bel](x) = m (x) + \frac{1 - k_{Bel}}{k_{Pl}} Pl (x) = k_{Bel} \tilde{Bel} (x) + (1 - k_{Bel}) \tilde{Pl} (x).
\end{equation}
More to the point, as its definition only involves belief and plausibility values \emph{of singletons}, it is more correct to think of $PraPl[b]$ as of a probability transformation of a probability interval system (rather than an approximation of a belief function)
\[
PraPl[(l,u)] \doteq l(x) + \frac{1 - \sum_y l(y)}{\sum_y u(y)} u(x)
\]
just like the intersection probability.

However, it is easier to point out its weakness as a representative of probability intervals when put in the above form. Just like in the case of relative belief and plausibility of singletons, $PraPl[(l,u)]$ is not in general consistent with the original probability interval system $(l,u)$.\\ If there exists an element $x\in\Theta$ such that $l(x) = u(x)$ (the interval has width $\Delta(x)$ equal to zero for that element) we have that
\[
\begin{array}{lll}
PraPl[(l,u)](x) & = & \displaystyle l(x) + \frac{1 - \sum_y l(y)}{\sum_y u(y)} u(x) = u(x) + \frac{1 - \sum_y l(y)}{\sum_y u(y)} u(x) \\ & = & \displaystyle u(x) \cdot \frac{\sum_y u(y) + 1 - \sum_y l(y)}{\sum_y u(y)} > u(x)
\end{array}
\]
as $\frac{\sum_y u(y) + 1 - \sum_y l(y)}{\sum_y u(y)} > 1$, and $PraPl[(l,u)]$ falls outside the interval.

Another fundamental objection against $PraPl[(l,u)]$ arises when we compare it to $p[(l,u)]$. While the latter adds to the lower bound $l(x)$ an equal fraction of the uncertainty $u(x) - l(x)$ for all singletons (\ref{eq:form2}), $PraPl[(l,u)]$ adds to the lower bound $l(x)$ an equal fraction \emph{of the upper bound} $u(x)$, effectively counting twice the evidence represented by the lower bound $l(x)$ (\ref{eq:prapl-interval}). 

In the case of belief functions, this amounts to adding to the mass value $m(x)$ of $x$ yet another fraction of $m(x)$ itself, instead of distributing only the remaining mass $Pl (x) - m (x)$ allowed to be assigned to $x$.

\section{Geometric interpretation} \label{sec:geometric}

Despite having being defined as a representative for systems of probability intervals, the intersection probability was first identified in the context of the geometric analysis of belief measures \cite{cuzzolin2008geometric,cuzzolin01thesis,cuzzolin01space,cuzzolin14lap}. As we briefly recall here, its very name derives from its geometry in the space of belief functions, or \emph{belief space} \cite{cuzzolin2008geometric,cuzzolin02fsdk}.

\subsection{Geometry in the belief space}

\subsubsection{Belief space}

Given a frame of discernment $\Theta$, a belief function $Bel : 2^\Theta \rightarrow [0,1]$ is completely specified by its $N - 2$ belief values $\{ Bel(A), \emptyset \subsetneq A \subsetneq \Theta \}$, $N \doteq 2^{|\Theta|}$ (as $Bel(\emptyset) = 0$, $Bel(\Theta) = 1$ for all BFs), and can then be seen as a point of $\mathbb{R}^{N-2}$. 

The \emph{belief space} associated with $\Theta$ is the set of points $\mathcal{B} \subset \mathbb{R}^{N-2}$ which correspond to admissible belief functions \cite{cuzzolin2008geometric}. This turns out to be the simplex determined by the convex closure of all the categorical belief functions $Bel_A$, namely
\[
\mathcal{B} = Cl(Bel_A,\; \emptyset \subsetneq A \subseteq \Theta),
\]
($Bel_\Theta$ included). The \emph{faces} of a simplex are all the simplices generated by a subset of its vertices. The set of all the Bayesian belief functions on $\Theta$, $\mathcal{P} = Cl(Bel_x, x\in\Theta)$, is then a face of $\mathcal{B}$.

Plausibility functions, also determined by their $N - 2$ values $\{ Pl(A), \emptyset \subsetneq A \subsetneq \Theta \}$, can too be seen as points of $\mathbb{R}^{N-2}$. We call \emph{plausibility space} \cite{cuzzolin2008dual,cuzzolin08pricai-moebius} the corresponding region $\mathcal{PL}$ of $\mathbb{R}^{N-2}$, again, a simplex \cite{cuzzolin03isipta}.

\begin{figure}[ht!]
\begin{center}
\includegraphics[width=0.55\textwidth]{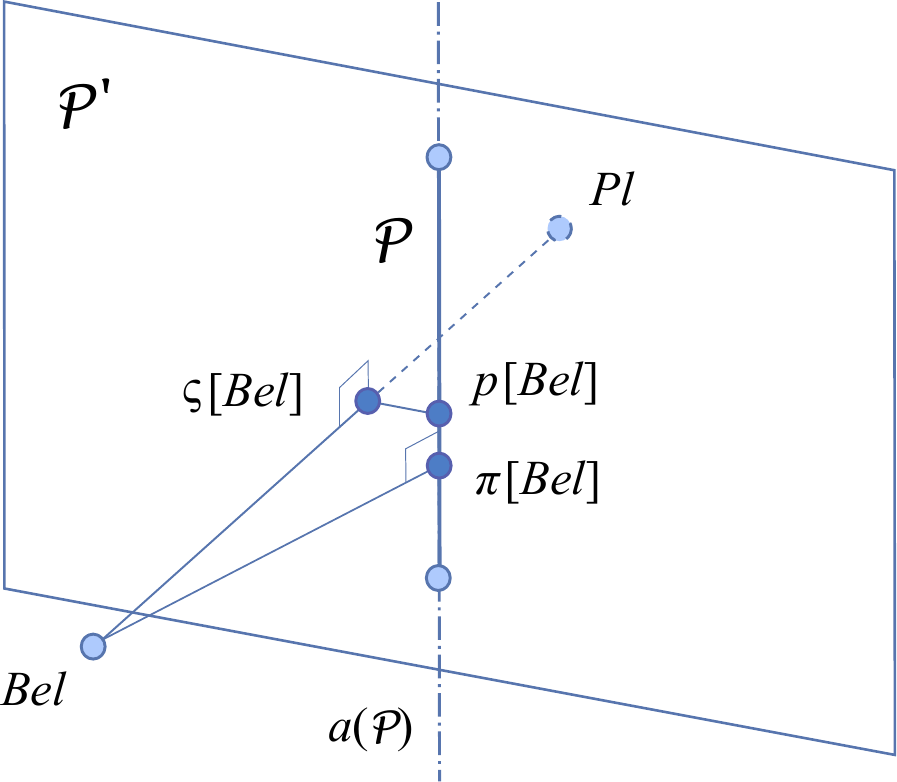}
\end{center}
\caption{The geometry of the line $a(Bel,Pl)$ and the relative locations of $p[Bel]$, $\varsigma[Bel]$ and of the \emph{orthogonal projection} $\pi[Bel]$ \cite{cuzzolin07smcb} for a frame of discernment of arbitrary size. Each belief function $Bel$ and the related plausibility function $Pl$ lie on opposite sides of the hyperplane $\mathcal{P}'$ of all Bayesian \emph{pseudo belief functions} \cite{cuzzolin07smcb}, which divides the space $\mathbb{R}^{N-2}$ of all such functions into two halves. The line $a(Bel,Pl)$ connecting them always intersects $\mathcal{P}'$, but not necessarily $a(\mathcal{P})$ (vertical line). This intersection $\varsigma[Bel]$ is naturally associated with a probability $p[Bel]$ (in general distinct from the orthogonal projection $\pi[Bel]$ of $Bel$ onto
$\mathcal{P}$), having the same components in the base $\{Bel_x, x \in \Theta\}$ of $a(\mathcal{P})$. $\mathcal{P}$ is a simplex (a segment in the figure) in $a(\mathcal{P})$: $\pi[Bel]$ and $p[Bel]$ are both `true' probabilities.
\label{fig:projections}}
\end{figure}

\subsubsection{Dual line}

The intersection probability for a belief function $Bel$ can be shown to be the unique probability distribution determined by the intersection of the \emph{dual line} joining a belief function $Bel$ and the related plausibility function $Pl$ with the region of Bayesian (pseudo) belief functions \cite{cuzzolin07smcb}.

It can be shown that this dual line $a(Bel,Pl)$ is always {orthogonal} to $\mathcal{P}$, but it {does not} intersect the probabilistic subspace in general. It does always intersect, however, the region of Bayesian pseudo belief functions (or `normalised sum functions') in a point
\begin{equation} \label{eq:varsigma}
\varsigma[Bel] \doteq Bel + \beta[Bel](Pl - Bel) = a(Bel,Pl) \cap \mathcal{P}'
\end{equation}
(where $\mathcal{P}'$ denotes the set of all Bayesian normalised sum functions in $\mathbb{R}^{N-2}$). $\varsigma[Bel]$ is a Bayesian pseudo BF but is not guaranteed to be a `proper' Bayesian {belief function}. 

But, of course, since $\sum_x m_{\varsigma[Bel]}(x) = 1$, $\varsigma[Bel]$ is naturally associated with a Bayesian belief function assigning an equal amount of mass to each singleton and 0 to each $A : |A|>1$. Namely, we can define the probability measure
\begin{equation} \label{eq:sigma}
p[Bel] \doteq \sum_{x\in\Theta} m_{\varsigma[Bel]}(x) Bel_x,
\end{equation}
where $m_{\varsigma[Bel]}(x)$ is given by 
\begin{equation} \label{eq:masses-varsigma}
m_{\varsigma[Bel]} (x) = m(x) + \beta[Bel] \sum_{A\supsetneq x} m(A).
\end{equation}

This Bayesian BF $p[Bel]$ is nothing but the intersection probability associated with the probability interval system induced by $Bel$.
The relative geometry of $\varsigma[Bel]$ and $p[Bel]$ with respect to the regions of Bayesian belief and normalised sum functions, respectively, is outlined in Fig. \ref{fig:projections}. 

\subsection{Justification for the name} \label{sec:justification}

The pseudo probability $\varsigma[Bel]$ (\ref{eq:varsigma}) provides the justification for the name `intersection probability'. It turns out that $p[Bel]$ and $\varsigma[Bel]$ are \emph{equivalent} when combined with a Bayesian belief function. 

We first need to recall the following result \cite{cuzzolin04smcb}.
\begin{proposition}
The orthogonal sum $Bel \oplus (\alpha_1 Bel_1 + \alpha_2 Bel_2)$ of a belief function $Bel$ and any affine combination $\alpha_1 Bel_1 + \alpha_2 Bel_2$, $\alpha_1 + \alpha_2 = 1$ of other two belief functions $Bel_1$, $Bel_2$ on the same frame reads as
\begin{equation} \label{eq:smcb}
Bel \oplus (\alpha_1 Bel_1 + \alpha_2 Bel_2) = \gamma_1 (Bel \oplus Bel_1) + \gamma_2 (Bel \oplus Bel_2),
\end{equation}
where
\[
\gamma_i = \frac{\alpha_i k(Bel, Bel_i)}{\alpha_1 k (Bel, Bel_1) + \alpha_2 k (Bel, Bel_2)}
\]
and $k(Bel, Bel_i)$ is the normalisation factor of the orthogonal sum $Bel \oplus Bel_i$.
\end{proposition}
Similar results can be proven for both conjunctive and disjunctive rules.
\begin{lemma} \label{lem:commuta}
Affine combination commutes with both conjunctive and disjunctive rules:
\[
\begin{array}{l}
Bel \ocap (\alpha_1 Bel_1 + \alpha_2 Bel_2) = \alpha_1 (Bel \ocap Bel_1) + \alpha_2 (Bel \ocap Bel_2), \\
Bel \ocup (\alpha_1 Bel_1 + \alpha_2 Bel_2) = \alpha_1 (Bel \ocup Bel_1) + \alpha_2 (Bel \ocup Bel_2)
\end{array}
\]
whenever $\alpha_1 + \alpha_2 = 1$.
\end{lemma}
\begin{proof} 
By definition (\ref{eq:disjunctive}), we have that $Bel \ocap (\alpha_1 Bel_1 + \alpha_2 Bel_2)$ has basic probability assignment:
\[
\begin{array}{lll}
& & m_{Bel \ocap (\alpha_1 Bel_1 + \alpha_2 Bel_2)}(A) 
\\ \\
& = & \displaystyle \sum_{B \cap C = A} m(B) m_{\alpha_1 Bel_1 + \alpha_2 Bel_2} (C) \\ & = & \displaystyle \sum_{B \cap C = A} m(B) \big ( \alpha_1 m_1(C) + \alpha_2 m_2(C) \big ) \\ & = & \displaystyle \alpha_1 \sum_{B \cap C = A} m(B) m_1 (C) + \alpha_2 \sum_{B \cap C = A} m (B) m_2 (C) 
\\ \\
& = & \alpha_1 m_{Bel \ocap Bel_1}(A) + \alpha_2 m_{Bel \ocap Bel_2}(A).
\end{array}
\]
An analogous proof holds for (\ref{eq:conjunctive}).
\end{proof} 

\begin{theorem} \label{the:ortpdb}
The combinations of $p[Bel]$ and $\varsigma[Bel]$ with any probability function $p \in \mathcal{P}$ coincide under both the Dempster (\ref{eq:dempster}) and conjunctive (\ref{eq:disjunctive}) rules:
\begin{equation} \label{eq:repre}
\begin{array}{ccc}
p[Bel] \oplus p = \varsigma[Bel] \oplus p, \quad & p[Bel] \cap p = \varsigma[Bel] \cap p, & \forall p\in\mathcal{P}.
\end{array}
\end{equation}
\end{theorem}
\begin{proof} 
Let us define by 
\begin{equation} \label{eq:bpla}
\mu(A) = \sum_{B\subseteq A} (-1)^{|A-B|} Pl(B) 
\end{equation}
the Moebius inverse of a plausibility function $Pl$ (see \cite{cuzzolin08pricai-moebius}). It can be proven that \cite{cuzzolin10ida}:
\begin{equation}\label{eq:diamond}
\sum_{A\supseteq \{x\}} \mu (A) = m (x).
\end{equation}
Now, applying Equation (\ref{eq:smcb}) to $\varsigma \oplus p$ yields
\begin{equation}\label{eq:arrow}
\begin{array}{lll}
\varsigma \oplus p 
& = & 
\big[ \beta[Bel] Pl + (1 - \beta[Bel]) Bel \big] \oplus p 
\\ \\
& = & 
\displaystyle \frac{\beta[Bel] k(p, Pl) Pl \oplus p + (1 - \beta[Bel]) k(p, Bel) Bel \oplus p}{\beta[Bel] k(p, Pl) + (1 - \beta[Bel]) k(p, Bel)}
\end{array}
\end{equation}
where
\[
\begin{array}{llll}
k(p, Pl) & = & \displaystyle \sum_{x\in\Theta} p(x) \left ( \sum_{A \supseteq \{x\}} \mu (A) \right ) & = \displaystyle \sum_{x \in\Theta} p(x) m (x), 
\\ \\
k(p, Bel) & = & \displaystyle \sum_{x \in \Theta} p(x) \left ( \sum_{A \supseteq \{x\}} m (A) \right ) & \displaystyle = \sum_{x \in \Theta} p(x) Pl(x).
\end{array}
\]
by Equation (\ref{eq:diamond}), and by the definition of the plausibility of singletons ($\sum_{A \supseteq \{x\}} m (A) = Pl (x)$).
On the other hand, recalling Equation (\ref{eq:relation}), we can write
\begin{equation} \label{eq:dualline}
p[Bel] = \beta[Bel] \bar{Pl} + (1 - \beta[Bel]) \bar{Bel}
\end{equation}
where we call the quantities
\begin{equation}\label{eq:pbarra}
\begin{array}{ccc}
\displaystyle \bar{Pl} \doteq \sum_{x \in \Theta} Pl (x) Bel_x 
& \quad & 
\displaystyle \bar{Bel} = \sum_{x \in \Theta} m (x) Bel_x
\end{array}
\end{equation}
\emph{plausibility of singletons} and \emph{belief of singletons}, respectively \cite{cuzzolin2008semantics,cuzzolin2008dual,cuzzolin10amai,CUZZOLIN2012786}, for sake of consistency of nomenclature. However, $\bar{Pl}$ is traditionally referred to as the \emph{contour function}.

Therefore when we apply (\ref{eq:smcb}) to $p[Bel] \oplus p$, instead, we get (by Equation (\ref{eq:dualline})):
\begin{equation}\label{eq:tmp}
\begin{array}{lll}
p[Bel] \oplus p & = & \big[ \beta[Bel] \bar{Pl} + (1 - \beta[Bel]) \bar{Bel} \big] \oplus p 
\\ \\
& = & 
\displaystyle 
\frac{\beta[Bel] k(p, \bar{Pl}) \bar{Pl} \oplus p + (1 - \beta[Bel]) k(p, \bar{Bel}) \bar{Bel} \oplus p}{\beta[Bel] k(p, \bar{Pl}) + (1 - \beta[Bel]) k(p, \bar{Bel})}.
\end{array}
\end{equation}
By definition of Dempster's combination (\ref{eq:dempster}):
\[
\begin{array}{c}
\displaystyle 
\bar{Pl} \oplus p = \frac{\displaystyle \sum_{x \in \Theta} Bel_x p(x) ( Pl (x) + 1 -
k_{Pl})}{k(p, \bar{Pl})}, 
\\ \\
\displaystyle 
\bar{Bel} \oplus p
= 
\frac{ \displaystyle \sum_{x \in \Theta} Bel_x p(x) (m(x) + 1 - k_{Bel})}{k(p, \bar{Bel})}.
\end{array}
\]
Hence
\[
\begin{array}{ll}
\displaystyle k(p, \bar{Pl}) \bar{Pl} \oplus p 
& 
\displaystyle = \sum_{x \in \Theta} Bel_x p(x) Pl (x) + (1 - k_{Pl}) \sum_{x \in \Theta} Bel_x p(x) 
\\ 
& = 
\displaystyle 
k(Bel, p) Bel \oplus p + (1 - k_{Pl}) p, 
\\ \\
\displaystyle 
k(p, \bar{Bel}) \bar{Bel} \oplus p 
& 
\displaystyle = \sum_{x \in \Theta} Bel_x p(x) m (x) + (1 - k_{Bel}) \sum_{x \in \Theta} Bel_x p(x) \\ 
& = 
\displaystyle 
k(Pl, p) Pl \oplus p + (1 - k_{Bel}) p,
\end{array}
\]
as:
\begin{enumerate}
\item $\sum_{x \in \Theta} Bel_x p(x) = p$;
\item
in the calculation of $Bel \oplus p$, each singleton $x$ is assigned mass 
\[
p(x) \sum_{A \supseteq \{x\}} m (A) = p(x) Pl (x);
\]
\item
in the calculation of $Pl \oplus p$, each singleton $x$ is assigned mass 
\[
p(x) \sum_{A \supseteq \{x\}} \mu(A) = p(x) m(x) 
\]
(again by Equation (\ref{eq:diamond})).
\end{enumerate}
After replacing these expressions in the numerator of (\ref{eq:tmp}) we can notice that, as
\[
\begin{array}{ccc}
\displaystyle \beta[Bel] = \frac{1 - k_{Bel}}{k_{Pl} - k_{Bel}}, & & \displaystyle 1 - \beta[Bel] =
\frac{k_{Pl} - 1}{k_{Pl} - k_{Bel}},
\end{array}
\]
the contributions of $p$ vanish, leaving expression (\ref{eq:arrow}) for $\varsigma \oplus p$.

As conjunctive rule and affine combination commute, and $k(Bel_1, Bel_2) = 1$ for each pair of pseudo belief functions $Bel_1, Bel_2$ under conjunctive combination, the proof holds for $\ocap$ too.
\end{proof} 

Even though $p[Bel]$ is \emph{not} the actual intersection $\varsigma[Bel]$ of the line $a(Bel, Pl)$ with the region of pseudo probabilities in the belief space, it behaves exactly like it when aggregated to a probability distribution.

Notice that Theorem \ref{the:ortpdb} is \emph{not} a simple consequence of Voorbraak's representation theorem \cite{cuzzolin2008dual}:
\[
Bel \oplus p = \tilde{Pl} \oplus p.
\]
It is indeed easy to prove that the relative plausibility of singletons of $\varsigma[Bel]$ \emph{is not} $p[Bel]$. 
As $\varsigma[Bel](A) = Bel (A) + \beta[Bel] [Pl - Bel](A)$ we have:
\[
\begin{array}{lll}
& &
Pl_{\varsigma[Bel]} (x) 
\\
& = & 
\displaystyle 
1 - \varsigma[Bel](\{x\}^c) = 1 - Bel (\{x\}^c) - \beta[Bel] \left [ Pl (\{ x \}^c) - Bel (\{ x \}^c) \right ] 
\\ 
& = & 
\displaystyle 
Pl (x) - \beta[Bel] \left [ 1 - Bel (x) - 1 + Pl (x) \right ] 
\\
& = & Pl (x) - \beta[Bel] \left [ Pl (x) - Bel (x) \right ] 
\\ 
& = & 
\beta[Bel] Bel (x) + (1 - \beta[Bel]) Pl (x) = \beta[Bel] m (x) + (1 - \beta[Bel]) Pl (x).
\end{array}
\]
Its normalization factor is
\[
\begin{array}{lll}
\sum_x Pl_{\varsigma[Bel]} (x) 
& = & 
\displaystyle
\sum_x Pl (x) - \beta[Bel]  \left [ \sum_x Pl (x) - \sum_x m (x) \right ] 
\\
& = &
 \displaystyle
\sum_x Pl (x) - 1 + \sum_x m (x).
\end{array}
\]
Clearly then
\[
\tilde{Pl}_{\varsigma[Bel]} (x) = \frac{Pl_{\varsigma[Bel]} (x)}{\sum_y Pl_{\varsigma[Bel]} (y)} = \frac{\beta[Bel] m (x) + (1 - \beta[Bel]) Pl (x)}{\sum_y Pl (y) - 1 + \sum_y m (y)}
\]
is different from $p[Bel](x) = \beta[Bel] Pl (x) + (1 - \beta[Bel]) m (x)$.

\section{Credal rationale} \label{sec:credal}

Probability interval systems admit a credal representation, which for intervals associated with belief functions is also strictly related to the credal set $\mathcal{P}[Bel]$ of all consistent probabilities \cite{cuzzolin2010credal, cuzzolin2021springer}.

By the definition (\ref{eq:consistent}) of $\mathcal{P}[Bel]$, it follows that the polytope of consistent probabilities can be decomposed into a number of component polytopes, namely
\begin{equation} \label{eq:decomposition}
\mathcal{P}[Bel] = \bigcap_{i=1}^{n-1} \mathcal{P}^i[Bel],
\end{equation}
where $\mathcal{P}^i[Bel]$ is the set of probabilities that satisfy the lower probability constraint {for size-$i$ events},
\begin{equation} \label{eq:polytopes-i}
\mathcal{P}^i[Bel] \doteq \Big \{ P \in \mathcal{P} : P(A) \geq Bel(A), \forall A : |A| = i \Big \}.
\end{equation}
Note that for $i=n$ the constraint is trivially satisfied by all probability measures $P$: $\mathcal{P}^n[Bel] = \mathcal{P}$. 

\subsection{Lower and upper simplices} \label{sec:lower-upper-simplices}

A simple and elegant geometric description of interval probability systems can be provided if, instead of considering the polytopes (\ref{eq:polytopes-i}), we focus on the credal sets
\[
T^i[Bel] \doteq \Big \{ P' \in \mathcal{P}' : P'(A) \geq Bel(A), \forall A : |A| = i \Big \}.
\]
Here $\mathcal{P}'$ denotes the set of all \emph{pseudo-probability measures} $P'$ on $\Theta$, whose distribution $p' : \Theta \rightarrow \mathbb{R}$ satisfy the normalisation constraint $\sum_{x\in\Theta} p'(x) = 1$ but not necessarily the non-negativity one -- there may exist an element $x$ such that $p'(x)<0$.
In particular, we focus here on the set of pseudo-probability measures which satisfy the lower constraint on singletons,
\begin{equation} \label{eq:lower-simplex}
T^1[Bel] \doteq \Big \{ p' \in \mathcal{P}' : p'(x) \geq Bel(x) \;\; \forall x \in \Theta \Big \},
\end{equation}
and the set $T^{n-1}[Bel]$ of pseudo-probability measures which satisfy the analogous constraint on events of size $n-1$,
\begin{equation} \label{eq:upper-simplex3}
\begin{array}{lll}
T^{n-1}[Bel] & \doteq & \Big \{ P' \in \mathcal{P}' : P'(A) \geq Bel(A) \;\; \forall A: |A| = n-1 \Big \} \\ \\
& = & \Big \{ P' \in \mathcal{P}' : P'(\{x\}^c) \geq Bel(\{x\}^c) \;\; \forall x \in \Theta \Big \} 
\\ \\
& = & \Big \{ p' \in \mathcal{P}' : p'(x) \leq Pl(x) \;\; \forall x \in \Theta \Big \},
\end{array}
\end{equation}
i.e., the set of pseudo-probabilities which satisfy the upper constraint on singletons. 

\subsection{Simplicial form} 

The extension to pseudo-probabilities allows us to prove that the credal sets (\ref{eq:lower-simplex}) and (\ref{eq:upper-simplex3}) have the form of simplices (see \cite{cuzzolin2010credal} and \cite{cuzzolin2021springer}, Chapter 16).
\begin{theorem} \label{the:simplices}
The credal set $T^1[Bel]$, or \emph{lower simplex}, can be written as
\begin{equation} \label{eq:t1-simplex}
T^1[Bel] = Cl(t^1_x[Bel], x \in \Theta),
\end{equation}
namely as the convex closure of the vertices
\begin{equation} \label{eq:t1-vertices}
t^1_x[Bel] = \sum_{y \neq x} m(y) Bel_y + \left ( 1 - \sum_{y \neq x} m(y) \right ) Bel_x.
\end{equation}
Dually, the \emph{upper simplex} $T^{n-1}[Bel]$ reads as the convex closure
\begin{equation} \label{eq:tn-1-simplex}
T^{n-1}[Bel] = Cl(t^{n-1}_x[Bel], x \in \Theta)
\end{equation}
of the vertices
\begin{equation} \label{eq:tn-1-vertices}
t^{n-1}_x[Bel] = \sum_{y \neq x} Pl(y) Bel_y + \left ( 1 - \sum_{y \neq x} Pl(y) \right ) Bel_x.
\end{equation}
\end{theorem}

By (\ref{eq:t1-vertices}), each vertex $t^1_x[Bel]$ of the lower simplex is a pseudo-probability that adds the total mass $1 - k_{Bel}$ of non-singletons to that of the element $x$, leaving all the others unchanged:
\[
m_{t^1_x[Bel]} (x) = m(x) + 1 - k_{Bel}, 
\quad 
m_{t^1_x[Bel]} (y) = m(y) \; \forall y \neq x. 
\]
In fact, as $m_{t^1_x[Bel]} (z) \geq 0$ for all $z \in \Theta$ and for all $x \in \Theta$ (all $t^1_x[Bel]$ are actual probabilities), we have that
\begin{equation} \label{eq:equality-credal}
T^1[Bel] = \mathcal{P}^1[Bel],
\end{equation}
and $T^1[Bel]$ is completely included in the probability simplex.

On the other hand, the vertices (\ref{eq:tn-1-vertices}) of the upper simplex are not guaranteed to be valid probabilities. 
\\
Each vertex $t^{n-1}_x[Bel]$ assigns to each element of $\Theta$ different from $x$ its plausibility $Pl(y) = pl(y)$, while it subtracts from $Pl(x)$ the `excess' plausibility $k_{Pl} - 1$:
\[
\begin{array}{llll}
m_{t^{n-1}_x[Bel]} (x) & = & Pl(x) + ( 1 - k_{Pl} ), &
\\ 
m_{t^{n-1}_x[Bel]} (y) & = & Pl(y) & \forall y \neq x.
\end{array}
\]
Now, as $1 - k_{Pl}$ can be a negative quantity, $m_{t^{n-1}_x[Bel]} (x)$ can also be negative and $t^{n-1}_x[Bel]$ is not guaranteed to be a `true' probability. 

We will have confirmation of this fact in the example in Section \ref{sec:ternary-credal}. 

\subsection{Lower and upper simplices and probability intervals} 

By comparing (\ref{eq:credal-interval}), (\ref{eq:lower-simplex}) and (\ref{eq:upper-simplex3}), it is clear that the credal set $\mathcal{P}[(l,u)]$ associated with a set of probability intervals $(l,u)$ is nothing but the intersection 
\[
\mathcal{P}[(l,u)] = T[l] \cap T[u]
\]
of the lower and upper simplices (\ref{upper-lower-lu}) associated with its lower- and upper-bound constraints, respectively:
\begin{equation} \label{upper-lower-lu}
T[l] \doteq \Big \{ p : p(x) \geq l(x) \; \forall x \in \Theta \Big \}, 
\quad
T[u] \doteq \Big \{ p : p(x) \leq u(x) \;\; \forall x \in \Theta \Big \}.
\end{equation}
In particular, when these lower and upper bounds are those enforced by a pair of belief and plausibility measures on the singleton elements of a frame of discernment, $l(x) = Bel(x)$ and $u(x) = Pl(x)$, we get
\[
\mathcal{P}[(Bel,Pl)] = T^1[Bel] \cap T^{n-1}[Bel]. 
\]

\subsection{Ternary case} \label{sec:ternary-credal} 

Let us consider the case of a frame of cardinality 3, $\Theta = \{x,y,z\}$, and a belief function $Bel$ with mass assignment
\begin{equation} \label{eq:example-c}
\begin{array}{ccc}
m(x) = 0.2, & m(y) = 0.1, & m(z) = 0.3, 
\\ \\
m(\{x,y\}) = 0.1, \quad & m(\{y,z\}) = 0.2, \quad & m(\Theta) = 0.1.
\end{array}
\end{equation}
Figure \ref{fig:intersection-focus} illustrates the geometry of the related credal set $\mathcal{P}[Bel]$ in the simplex, denoted by $Cl(P_x,P_y,P_z)$, of all the probability measures on $\Theta$. 

It is well known \cite{Chateauneuf89,cuzzolin08-credal} that the credal set associated with a belief function is a polytope whose vertices are associated with all possible permutations of singletons.
\begin{proposition} \label{pro:vertices}
Given a belief function $Bel:2^\Theta \rightarrow [0,1]$, the simplex $\mathcal{P}[Bel]$ of the probability measures consistent with $Bel$ is the polytope
\[
\mathcal{P}[Bel] = Cl(P^\rho[Bel] \;\forall \rho),
\]
where $\rho$ is any permutation $\{ x_{\rho(1)}, \ldots, x_{\rho(n)} \}$ of the singletons of $\Theta$, and the vertex $P^\rho[Bel]$ is the Bayesian belief function such that
\begin{equation} 
P^\rho[Bel](x_{\rho(i)}) = \sum_{A \ni x_\rho(i), A \not\ni x_\rho(j) \; \forall j<i} m(A).
\end{equation}
\end{proposition}

By Proposition \ref{pro:vertices}, for the example belief function (\ref{eq:example-c}), $\mathcal{P}[Bel]$ has as vertices the probabilities $P^{\rho^1}, P^{\rho^2}, P^{\rho^3}$, $P^{\rho^4}$, $P^{\rho^5}[Bel]$ identified by purple squares in Fig. \ref{fig:intersection-focus}, namely
\begin{equation} \label{eq:example-rho-credal}
\begin{array}{llll}
\rho^1 = (x,y,z): & \quad P^{\rho^1}[Bel](x) = .4, & P^{\rho^1}[Bel](y) = .3, & P^{\rho^1}[Bel](z) = .3,
\\ 
\rho^2 = (x,z,y): & \quad P^{\rho^2}[Bel](x) = .4, & P^{\rho^2}[Bel](y) = .1, & P^{\rho^2}[Bel](z) = .5, 
\\ 
\rho^3 = (y,x,z): & \quad P^{\rho^3}[Bel](x) = .2, & P^{\rho^3}[Bel](y) = .5, & P^{\rho^3}[Bel](z) = .3, 
\\ 
\rho^4 = (z,x,y): & \quad P^{\rho^4}[Bel](x) = .3, & P^{\rho^4}[Bel](y) = .1, & P^{\rho^4}[Bel](z) = .6, 
\\ 
\rho^5 = (z,y,x): & \quad P^{\rho^5}[Bel](x) = .2, & P^{\rho^5}[Bel](y) = .2, & P^{\rho^5}[Bel](z) = .6
\end{array}
\end{equation}
(as the permutations $( y,x,z )$ and $( y,z,x )$ yield the same probability distribution). 

\begin{figure}[ht!] 
\centering
\includegraphics[width = 10cm]{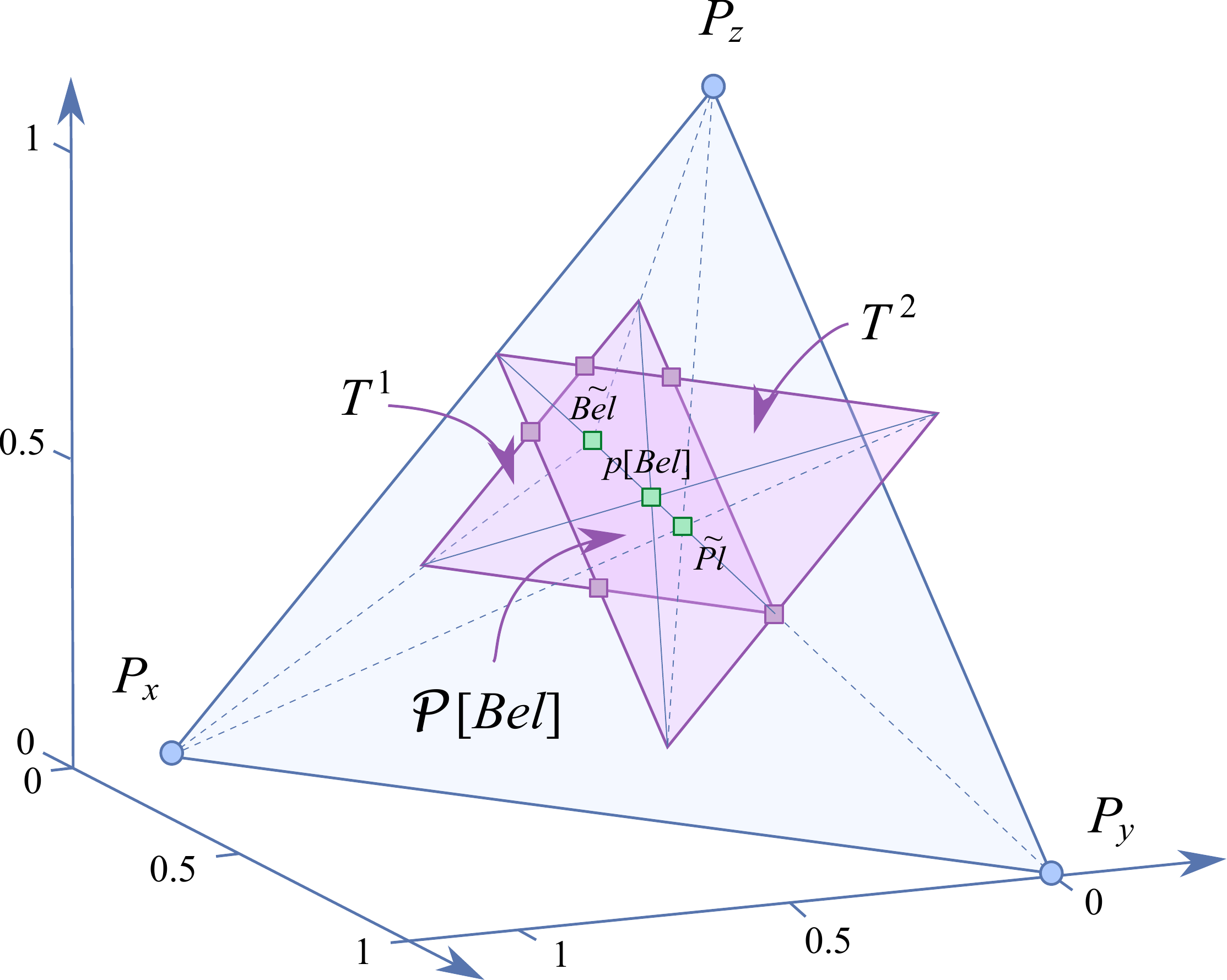}
\caption{The polytope $\mathcal{P}[Bel]$ of probabilities consistent with the belief function $Bel$ (\ref{eq:example-c}) defined on $\{x,y,z\}$ is shown in a darker shade of purple, in the probability simplex $\mathcal{P} = Cl(P_x,P_y,P_z)$ (shown in light blue). Its vertices (purple squares) are given in (\ref{eq:example-rho-credal}). The intersection probability, relative belief and relative plausibility of singletons (green squares) are the foci of the pairs of simplices $\{ T^1[Bel],T^2[Bel] \}$, $\{ T^1[Bel],\mathcal{P} \}$ and $\{ \mathcal{P},T^2[Bel] \}$, respectively. In the ternary case, $T^1[Bel]$ and $T^2[Bel]$ (light purple) are regular triangles. Geometrically, their focus is the intersection of the lines joining their corresponding vertices (dashed lines for $\{T^1[Bel],\mathcal{P}\}$,$\{\mathcal{P},T^2[Bel]\}$; solid lines for $\{T^1[Bel],T^2[Bel]\}$). \label{fig:intersection-focus}}
\end{figure}

We can notice a number of relevant facts:
\begin{enumerate}
\item
As pointed out above, $\mathcal{P}[Bel]$ (the polygon delimited by the purple squares) is the intersection of the two triangles (two-dimensional simplices) $T^1[Bel]$ and $T^2[Bel]$.
\item
The relative belief of singletons,
\[
\tilde{Bel}(x) = \frac{.2}{.6} = \frac{1}{3}, 
\quad 
\tilde{Bel}(y) = \frac{.1}{.6} = \frac{1}{6}, 
\quad
\tilde{Bel}(z) = \frac{.3}{.6} = \frac{1}{2},
\]
is the {intersection of the lines joining the corresponding vertices of the probability simplex $\mathcal{P}$ and the lower simplex $T^1[Bel]$}.
\item
The relative plausibility of singletons,
\[
\begin{array}{lll}
\tilde{Pl}(x) & = & \displaystyle \frac{m(x) + m(\{x,y\}) + m(\Theta)}{k_{Pl} - k_{Bel}} = \frac{.4}{.4+.5+.6} = \frac{4}{15}, 
\\ \\
\tilde{Pl}(y) & = & \displaystyle \frac{.5}{.4+.5+.6} = \frac{1}{3}, 
\quad 
\tilde{Pl}(z) = \frac{2}{5},
\end{array}
\]
is the {intersection of the lines joining the corresponding vertices of the probability simplex $\mathcal{P}$ and the upper simplex $T^2[Bel]$}.
\item
Finally, the intersection probability,
\[
\begin{array}{lll}
p[Bel](x) & = & \displaystyle
m(x) + \beta[Bel] \big ( m(\{x,y\}) + m(\Theta) \big ) = .2 + \frac{.4 \ast 0.2}{1.5 - 0.4} = .27,
\\ 
p[Bel](y) & = & \displaystyle
 .1 + \frac{.4}{1.1} 0.4 = .245, \quad p[Bel](z) = .485,
\end{array}
\]
is the unique intersection of the lines joining the corresponding vertices of the upper and lower simplices $T^2[Bel]$ and $T^1[Bel]$.
\end{enumerate}

Although Fig. \ref{fig:intersection-focus} suggests that $\tilde{Bel}$, $\tilde{Pl}$ and $p[Bel]$ might be consistent with $Bel$, this is a mere artefact of this ternary example, for it can be proved (\cite{cuzzolin2021springer}, Chapter 12) that neither the relative belief of singletons nor the relative plausibility of singletons necessarily belongs to the credal set $\mathcal{P}[Bel]$. 
Indeed, the point here is that the epistemic transforms $\tilde{Bel}$, $\tilde{Pl}$, $p[Bel]$ are instead {consistent with the interval probability system $\mathcal{P}[(Bel,Pl)]$ associated with the original belief function $Bel$}:
\[
\tilde{Bel}, \tilde{Pl}, p[Bel] \in \mathcal{P}[(Bel,Pl)] = T^1[Bel] \cap T^{n-1}[Bel].
\]

Their geometric behaviour, as described by facts 2, 3 and 4, holds in the general case as well.

\subsection{Focus of a pair of simplices} \label{sec:focus-simplices} 

\begin{definition} \label{def:focus}
Consider two simplices in $\mathbb{R}^{n-1}$, denoted by $S = Cl(s_1, \ldots, s_n)$ and $T = Cl(t_1, \ldots, t_n)$, with the same number of vertices. If
there exists a permutation $\rho$ of $\{1, \ldots, n\}$ such that the intersection
\begin{equation} \label{eq:lines}
\bigcap_{i=1}^n a(s_i,t_{\rho(i)})
\end{equation}
of the lines joining corresponding vertices of the two simplices
exists and is unique, then $p = f(S,T) \doteq \bigcap_{i=1}^n a(s_i,t_{\rho(i)})$ is termed the \emph{focus} of the two simplices $S$ and $T$.
\end{definition} 

Not all pairs of simplices admit a focus. For instance, the pair of simplices (triangles) $S = Cl(s_1=[2,2]',s_2=[5,2]', s_3 = [3,5]')$ and $T = Cl( t_1 = [3,1]', t_2 = [5,6]', t_3 = [2,6]' )$ in $\mathbb{R}^2$ does not admit a focus, as no matter what permutation of the order of the vertices we consider, the lines joining corresponding vertices do not intersect.
Geometrically, all pairs of simplices admitting a focus can be constructed by considering all possible stars of lines, and all possible pairs of points on each line of each star as pairs of corresponding vertices of the two simplices. 

\begin{definition} \label{def:focus-special}
We call a focus \emph{special} if the affine coordinates of $f(S,T)$ on the lines $a(s_i,t_{\rho(i)})$ all coincide, namely $\exists \alpha \in \mathbb{R}$ such that
\begin{equation} \label{eq:condition}
f(S,T) = \alpha s_i + (1 - \alpha) t_{\rho(i)} \quad \forall i = 1, \ldots, n.
\end{equation}
\end{definition}

Not all foci are special. As an example, the simplices $S = Cl(s_1=[-2,-2]', s_2 = [0,3]', s_3 = [1,0]')$ and $T = Cl( t_1 = [-1,0]', t_2 = [0,-1]', t_3 = [2,2]' )$ in $\mathbb{R}^2$ admit a focus, in particular for the permutation $\rho(1)=3, \rho(2)=2, \rho(3)=1$ (i.e., the lines $a(s_1,t_3)$, $a(s_2,t_2)$ and $a(s_3,t_1)$ intersect in $\mathbf{0} = [0,0]'$).
However, the focus $f(S,T) = [0,0]'$ is not special, as its simplicial coordinates in the three lines are $\alpha = \frac{1}{2}$, $\alpha = \frac{1}{4}$ and $\alpha = \frac{1}{2}$, respectively.

On the other hand, given any pair of simplices in $\mathbb{R}^{n-1}$ $S = Cl(s_1, \ldots, s_n)$ and $T = Cl(t_1, \ldots, t_n)$, for any permutation $\rho$ of the indices there always exists a linear variety of points which have the same affine coordinates in both simplices,
\begin{equation} \label{eq:focus}
\left \{ p \in \mathbb{R}^{n-1} \Bigg | p = \sum_{i = 1}^n \alpha_i s_i = \sum_{j = 1}^n \alpha_j t_{\rho(j)}, 
\;
\sum_{i = 1}^n \alpha_i = 1.
\right \}
\end{equation} 
Indeed, the conditions on the right-hand side of (\ref{eq:focus}) amount to a linear system of $n$ equations in $n$ unknowns ($\alpha_i$, $i=1, \ldots, n$). Thus, there exists a linear variety of solutions to such a system, whose dimension depends on the rank of the matrix of constraints in (\ref{eq:focus}).

It is rather easy to prove the following theorem.
\begin{theorem} \label{the:focus-intersection}
Any special focus $f(S,T)$ of a pair of simplices $S,T$ has the same affine coordinates in both simplices, i.e.,
\[
f(S,T) =
\sum_{i = 1}^n \alpha_i s_i = \sum_{j = 1}^n \alpha_j t_{\rho(j)}, 
\quad
\sum_{i = 1}^n \alpha_i = 1,
\]
where $\rho$ is the permutation of indices for which the intersection of the lines $a(s_i,t_{\rho(i)})$, $i=1, \ldots, n$ exists.
\end{theorem} 

Note that the affine coordinates $\{\alpha_i,i\}$ associated with a focus can be negative, i.e., the focus may be located outside one or both simplices.

Notice also that the barycentre itself of a simplex is a special case of a focus. In fact, the centre of mass $b$ of a $d$-dimensional simplex $S$ is the intersection of the medians of $S$, i.e., the lines joining each vertex with the barycentre of the opposite (($d-1$)-dimensional) face (see Fig. \ref{fig:bary}). But those barycentres, for all ($d-1$)-dimensional faces, themselves constitute the vertices of a simplex $T$.

\begin{figure}[ht!]
\begin{center}
\includegraphics[height = 5cm]{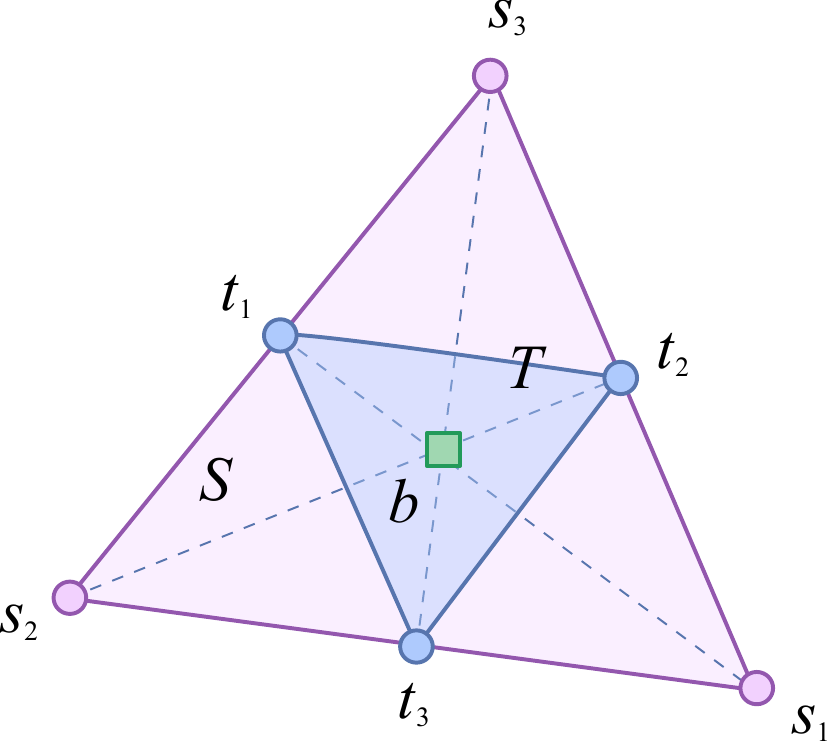}
\end{center}
\caption{The barycentre of a simplex is itself a special case of a focus, as shown here for a two-dimensional example. \label{fig:bary}}
\end{figure}

\subsection{Intersection probability as a focus} 

\begin{theorem} \label{the:pdb-focus}
For each belief function $Bel$, the intersection probability $p[Bel]$ has the same affine coordinates in the lower and upper simplices $T^1[Bel]$ and $T^{n-1}[Bel]$, respectively.
\end{theorem}

\begin{theorem} \label{the:pdb-coordinate}
The intersection probability is the special focus of the pair of lower and upper simplices $\{ T^1[Bel], T^{n-1}[Bel] \}$, with its affine coordinate on the corresponding intersecting lines equal to $\beta[Bel]$ (\ref{eq:beta}).
\end{theorem}
The fraction $\alpha = \beta[Bel]$ of the width of the probability interval that generates the intersection probability can be read in the probability simplex as its coordinate on any of the lines determining the special focus of $\{ T^1[Bel], T^{n-1}[Bel] \}$.

Similar results hold for the relative belief and relative plausibility of singletons, which are the (special) foci associated with the lower and upper simplices $T^1[Bel]$ and $T^{n-1}[Bel]$, the geometric incarnations of the lower and upper constraints on singletons. Those two simplices can also be interpreted as the sets of probabilities consistent with the plausibility and belief of singletons, respectively (see \cite{cuzzolin2021springer}, Chapter 12). 


Just as the pignistic function adheres to sensible rationality principles and, as a consequence, it has a clear geometrical interpretation as the centre of mass of the credal set associated with a belief function $Bel$, the intersection probability has an elegant geometric behaviour with respect to the credal set associated with an interval probability system, being the (special) focus of the related upper and lower simplices.

Now, selecting the special focus of two simplices representing two different constraints (i.e., the point with the same convex coordinates in the two simplices) means adopting the single probability distribution which satisfies both constraints {in exactly the same way}. If we assume homogeneous behaviour in the two sets of constraints $\{ P(x)\geq Bel(x) \; \forall x\}$, $\{ P(x) \leq Pl(x) \; \forall x\}$ as a rationality principle for the probability transformation of an interval probability system, then the intersection probability necessarily follows as the unique solution to the problem. 

Another interesting results stems from the fact that the pignistic function and affine combination commute:
\[
BetP[\alpha_1 Bel_1 + \alpha_2 Bel_2] = \alpha_1 BetP[Bel_1] + \alpha_2 BetP[Bel_2]
\]
whenever $\alpha_1 + \alpha_2 = 1$.
\begin{proposition} \cite{cuzzolin2021springer}
The intersection probability is the convex combination of the barycentres $t^1[Bel], t^{n-1}[Bel]$ of the lower and upper simplices $T^1[Bel], T^{n-1}[Bel]$ associated with a belief function $Bel$, with coefficient (\ref{eq:beta})
\[
p[Bel] = \beta[Bel] t^{n-1}[Bel] + (1 - \beta[Bel]) t^1[Bel].
\]
\end{proposition}

\section{Relations with other probability transforms} \label{sec:relations}

In Section \ref{sec:comparison-interval} we discussed how the intersection probability relates to alternative probability transforms for probability intervals, including Sudano's proposals.

Since the intersection probability can also be interpreted (as we have seen) as a probability transform which applies to belief functions, it is quite natural to wonder how it relates to the two main classes of such transforms: the epistemic family and the affine family (see \cite{cuzzolin2021springer}, Chapters 11 and 12).

\subsection{Epistemic transforms} \label{sec:relations-epistemic}

From (\ref{eq:form2}) it follows that, for a belief function $Bel$, 
\begin{equation}\label{eq:interpretation1}
p[Bel](x) = m(x) + (1 - k_{Bel}) \frac{Pl(x) - m(x)}{k_{Pl} - k_{Bel}},
\end{equation}
which can be rewritten as:
\begin{equation} \label{eq:convex-r}
p[Bel] = k_{Bel} \; \tilde{Bel} + (1 - k_{Bel}) R[Bel].
\end{equation}
Since $k_{Bel} = \sum_{x\in\Theta} m(x) \leq 1$, (\ref{eq:convex-r}) implies that the intersection probability $p[Bel]$ belongs to the segment linking the relative uncertainty of singletons $R[Bel]$ to the relative belief of singletons $\tilde{Bel}$. Its convex coordinate on this segment is the total mass of singletons $k_{Bel}$. 

The relative plausibility function $\tilde{Pl}$ can also be written in terms of $\tilde{Bel}$ and $R[Bel]$ as, by definition (\ref{eq:relplaus}),
\[
\begin{array}{lll}
R[Bel](x) & = & \displaystyle \frac{Pl(x) - m(x)}{k_{Pl} - k_{Bel}} = \frac{Pl(x)}{k_{Pl} - k_{Bel}} - \frac{m(x)}{k_{Pl} - k_{Bel}} 
\\ \\
& = & \displaystyle \tilde{Pl}(x) \frac{k_{Pl}}{k_{Pl} - k_{Bel}} - \tilde{Bel}(x) \frac{k_{Bel}}{k_{Pl} - k_{Bel}},
\end{array}
\]
since $\tilde{Pl}(x) = Pl(x)/k_{Pl}$ and $\tilde{Bel}(x) = m(x)/k_{Bel}$. Therefore,
\begin{equation} \label{eq:relplaus-on-segment}
\tilde{Pl} = \left ( \frac{k_{Bel}}{k_{Pl}} \right ) \tilde{Bel} + \left ( 1 -
\frac{k_{Bel}}{k_{Pl}} \right ) R[Bel].
\end{equation}

In summary, both the relative plausibility of singletons $\tilde{Pl}$ and the intersection probability $p[Bel]$ belong to the segment $Cl(R[Bel],\tilde{Bel})$ joining the relative belief $\tilde{Bel}$ and the relative uncertainty of singletons $R[Bel]$ (see Fig. \ref{fig:btildeR}).  
The convex coordinate of $\tilde{Pl}$ in $Cl(R[Bel],\tilde{Bel})$ (\ref{eq:relplaus-on-segment}) measures the ratio between the total mass and the total plausibility of the singletons, while that of $\tilde{Bel}$ measures the total mass of the singletons $k_{Bel}$.
Since $k_{Pl} = \sum_{A\subset\Theta} m(A)|A| \geq 1$, we have that $k_{Bel}/k_{Pl} \leq k_{Bel}$: hence the relative plausibility function of singletons $\tilde{Pl}$ is closer to $R[Bel]$ than $p[Bel]$ is (Figure \ref{fig:btildeR} again).
\begin{figure}[ht!]
\begin{center}
\includegraphics[width=0.7\textwidth]{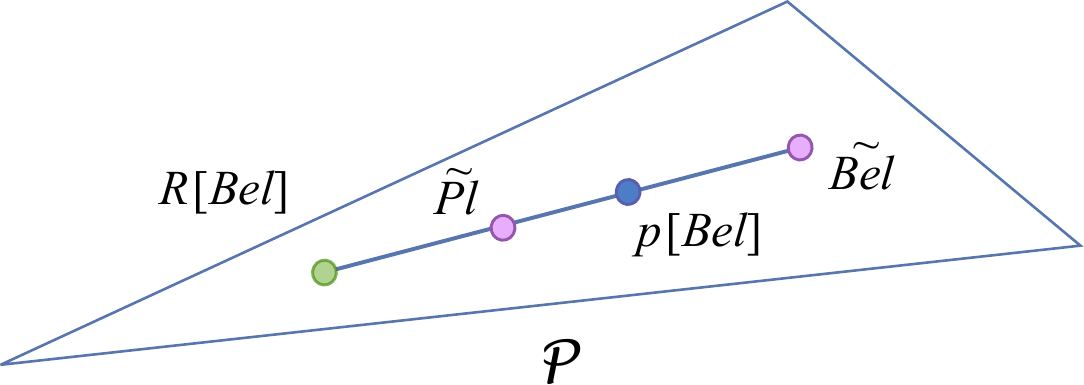}
\end{center}
\caption{\label{fig:btildeR} Location in the probability simplex $\mathcal{P}$ of the intersection probability $p[Bel]$ and the relative plausibility of singletons
$\tilde{Pl}$ with respect to the relative uncertainty of singletons $R[Bel]$. They both lie on the segment joining $R[Bel]$ and the relative belief of singletons $\tilde{Bel}$, but $\tilde{Pl}$ is closer to $R[Bel]$ than $p[Bel]$ is.}
\end{figure}

Obviously, when $k_{Bel} = 0$ (the relative belief of singletons $\tilde{Bel}$ does not exist, for $Bel$ assigns no mass to singletons), the remaining probability approximations coincide: $p[Bel] = \tilde{Pl} = R[Bel]$ by (\ref{eq:interpretation1}). 

\subsection{Pignistic function} \label{sec:relations-pignistic}

To shed even more light on $p[Bel]$, and to get an alternative interpretation of the intersection probability, it is useful to compare $p[Bel]$ as expressed in (\ref{eq:interpretation1}) with the pignistic function,
\[
BetP[Bel](x) \doteq \sum_{A \supseteq x} \frac{m(A)}{|A|} = m(x) + \sum_{A \supset x, A \neq x} \frac{m(A)}{|A|}.
\]
In $BetP[Bel]$ the mass of each event $A$, $|A|>1$, is considered \emph{separately}, and its mass $m(A)$ is shared \emph{equally} among the elements of $A$. In $p[Bel]$, instead, it is the \emph{total} mass $\sum_{|A|>1} m(A) = 1 - k_{Bel}$ of non-singleton focal sets which is considered, and this total mass is distributed \emph{proportionally} to their non-Bayesian contribution to each element of $\Theta$. 

Now, if $|A|>1$,
\[
\beta[Bel_A] = \frac{\sum_{|B|>1} m(B)}{\sum_{|B|>1} m(B) |B|} = \frac{1}{|A|},
\]
so that both $p[Bel](x)$ and $BetP[Bel](x)$ assume the form
\[
m(x) + \sum_{A \supset x, A \neq x} m(A) \beta_A,
\]
where $\beta_A = \text{const} = \beta[Bel]$ for $p[Bel]$, while $\beta_A = \beta[Bel_A]$ in the case of the pignistic function. 

Under what conditions do the intersection probability and pignistic function coincide? A sufficient condition
can be easily given for a special class of belief functions.
\begin{theorem} \label{the:sigmapign}
The intersection probability and pignistic function coincide for a given belief function $Bel$ whenever the focal elements of $Bel$ have size 1 or $k$ only.
\end{theorem}
\begin{proof}
The desired equality $p[Bel] = BetP[Bel]$ is equivalent to
\[
m(x) + \sum_{A\supsetneq x} m(A) \beta[Bel] = m(x) + \sum_{A\supsetneq x} \frac{m(A)}{|A|},
\]
which in turn reduces to
\[
\sum_{A\supsetneq x} m(A) \beta[Bel] = \sum_{A\supsetneq x} \frac{m(A)}{|A|}.
\]
If $\exists$ $k : m(A) = 0$ for $|A|\neq k$, $|A|>1$, then $\beta[Bel] = 1/k$ and the equality is satisfied.
\end{proof}

In particular, this is true when $Bel$ is 2-additive (its focal elements have cardinality $\leq 2$).

\begin{exam} \label{exa:example-pdb}
Let us briefly discuss these two interpretations of $p[Bel]$ in a simple example. Consider a ternary frame $\Theta = \{ x ,y,z \}$, and a belief function $Bel$ with BPA
\begin{equation} \label{eq:example-pdb-bf}
\begin{array}{c}
\begin{array}{ccc}
m(x) = 0.1, & m(y) = 0, & m(z) = 0.2,
\end{array}\\
\begin{array}{cccc}
m(\{x,y\}) = 0.3, & m(\{x,z\}) = 0.1, & m(\{y,z\}) = 0, & m(\Theta) = 0.3.
\end{array}
\end{array}
\end{equation}
The related basic plausibility assignment is, according to (\ref{eq:bpla}),
\[
\begin{array}{c}
\begin{array}{lll}
\mu(x) & = & \displaystyle (-1)^{|x|+1} \sum_{B\supseteq \{x\}} m(B) \\ & = & \displaystyle m(x) + m(\{x,y\}) + m(\{x,z\}) + m(\Theta) = 0.8, \end{array} \\ \\ \begin{array}{lllll} \mu(y) = 0.6, & \hspace{1mm} & \mu(z) = 0.6, & \hspace{1mm} & \mu(\{x,y\}) = - 0.6, \\ \\
\mu(\{x,z\}) = -0.4, & \hspace{1mm} & \mu(\{y,z\}) = - 0.3, & \hspace{1mm} & \mu(\Theta) = 0.3.
\end{array}
\end{array}
\]
Figure \ref{fig:exa-pdb} depicts the subsets of $\Theta$ with non-zero BPA (left) and BPlA (middle) induced by the belief function (\ref{eq:example-pdb-bf}): dashed ellipses indicate a negative mass. The total mass that (\ref{eq:example-pdb-bf}) accords to singletons is $k_{Bel} = 0.1 + 0 + 0.2 = 0.3$. Thus, the line coordinate $\beta[Bel]$ of the intersection $\varsigma[Bel]$ of the line $a(Bel,Pl)$ with $\mathcal{P}'$ is
\[
\beta[Bel] = \frac{1 - k_{Bel}}{m(\{x,y\})|\{x,y\}| + m(\{x,z\})|\{x,z\}| + m(\Theta)|\Theta|} = \frac{0.7}{1.7}.
\]
By (\ref{eq:masses-varsigma}), the mass assignment of $\varsigma[Bel]$ is therefore
\[
\begin{array}{l}
\displaystyle 
m_{\varsigma[Bel]}(x) = m(x) + \beta[Bel] (\mu(x) - m(x)) = 0.1 + 0.7 \cdot \frac{0.7}{1.7} = 0.388, 
\\ \\
\displaystyle 
m_{\varsigma[Bel]}(y) = 0 + 0.6 \cdot \frac{0.7}{1.7} = 0.247, \;\;\;\;\;\;\;\;\; m_{\varsigma[Bel]}(z) = 0.2 + 0.4 \cdot \frac{0.7}{1.7} = 0.365, 
\\ \\
\displaystyle 
m_{\varsigma[Bel]}(\{x,y\}) = 0.3 - 0.9 \cdot \frac{0.7}{1.7} = - 0.071, 
\\ \\
\displaystyle 
m_{\varsigma[Bel]}(\{x,z\}) = 0.1 - 0.5 \cdot \frac{0.7}{1.7} = - 0.106, 
\\ \\
\displaystyle 
m_{\varsigma[Bel]}(\{y,z\}) = 0 - 0.3 \cdot \frac{0.7}{1.7} = -0.123, \;\;\;\; m_{\varsigma[Bel]}(\Theta) = 0.3 + 0 \cdot \frac{0.7}{1.7} = 0.3.
\end{array}
\]
We can verify that all singleton masses are indeed non-negative and add up to one, while the masses of the non-singleton events add up to zero,
\[
- 0.071 - 0.106 -0.123 + 0.3 = 0, 
\]
confirming that $\varsigma[Bel]$ is a Bayesian normalised sum function (pseudo belief function). Its mass assignment has signs which are still described by Fig. \ref{fig:exa-pdb} (middle) although, as $m_{\varsigma[Bel]}$ is a weighted average of $m$ and $\mu$, its mass values are closer to zero.

\begin{figure}[ht!]
\begin{center}
\includegraphics[width=\textwidth]{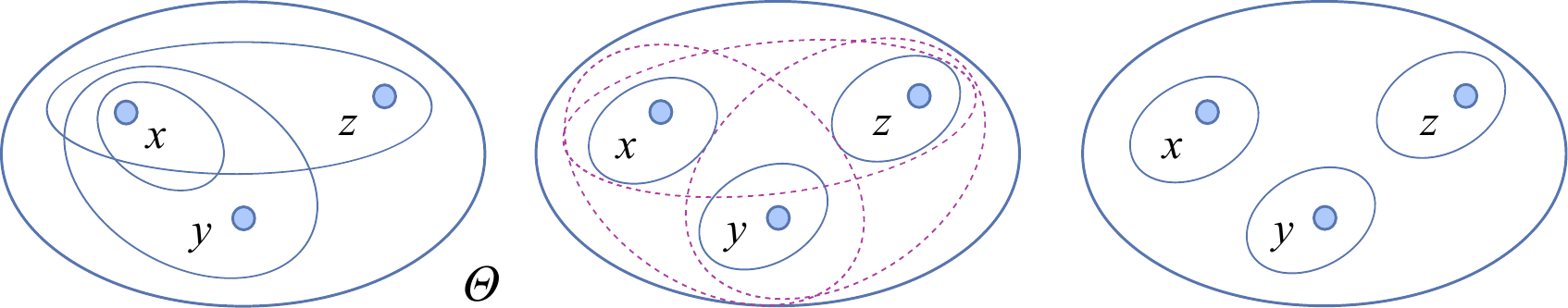}
\end{center}
\caption{Signs of non-zero masses assigned to events by the functions discussed in Example \ref{exa:example-pdb}. Left: BPA of the belief function (\ref{eq:example-pdb-bf}), with five focal elements. Middle: the associated BPlA assigns positive masses (solid ellipses) to all events of size 1 and 3, and negative ones (dashed ellipses) to all events of size 2. This is the case for the mass assignment associated with $\varsigma$ (\ref{eq:masses-varsigma}) too. Right: the intersection probability $p[Bel]$ (\ref{eq:sigma}) retains, among the latter, only the masses assigned to singletons. \label{fig:exa-pdb}}
\end{figure}

In order to compare $\varsigma[Bel]$ with the intersection probability, we need to recall (\ref{eq:interpretation1}): the non-Bayesian contributions of $x,y,z$ are, respectively,
\[
\begin{array}{lll}
Pl(x) - m(x) & = & m(\Theta) + m(\{x,y\}) + m(\{x,z\}) = 0.7, \\ Pl(y) - m(y) & = & m(\{x,y\}) + m(\Theta) = 0.6, \\ Pl(z) - m(z) & = & m(\{x,z\}) + m(\Theta) = 0.4,
\end{array}
\]
so that the relative uncertainty of singletons is 
\[
R(x) = 0.7/1.7, \quad R(y) = 0.6/1.7, \quad R(z) = 0.4/1.7. 
\]
For each singleton $\theta$, the value of the intersection probability results from adding to the original BPA $m(\theta)$ a share of the mass of the non-singleton events $1 - k_{Bel} = 0.7$ proportional to the value of $R(\theta)$ (see Fig. \ref{fig:exa-pdb}, right):
\[
\begin{array}{lllll}
p[Bel](x) & = & m(x) + (1 - k_{Bel}) R(x) = 0.1 + 0.7 * 0.7/1.7 & = & 0.388, 
\\
p[Bel](y) & = & m(y) + (1 - k_{Bel}) R(y) = 0 + 0.7 * 0.6/1.7 & = & 0.247, 
\\
p[Bel](z) & = & m(z) + (1 - k_{Bel}) R(z) = 0.2 + 0.7 * 0.4/1.7 & = & 0.365.
\end{array}
\]
We can see that $p[Bel]$ coincides with the restriction of $\varsigma[Bel]$ to singletons.

Equivalently, $\beta[Bel]$ measures the share of $Pl(x) - m(x)$ assigned to each element of the frame of discernment:
\[
\begin{array}{lllll}
p[Bel](x) & = & m(x) + \beta[Bel] (Pl(x) - m(x)) & = & 0.1 + 0.7/1.7 * 0.7, 
\\
p[Bel](y) & = & m(y) + \beta[Bel] (Pl(y) - m(y)) & = & 0 + 0.7/1.7 * 0.6, 
\\
p[Bel](z) & = & m(z) + \beta[Bel] (Pl(z) - m(z)) & = & 0.2 + 0.7/1.7 * 0.4.
\end{array}
\]
\end{exam} 

\section{Operators} \label{sec:operators}

To complete our analysis, it can be useful to understand the way the intersection probability relates to two major (geometric) operators in the space of belief functions: affine combination and convex closure.

\subsection{Affine combination} \label{sec:operators-affine}

We have seen in Section \ref{sec:relations} that $p[Bel]$ and $BetP[Bel]$ are closely related probability transforms, linked by the role of the quantity $\beta[Bel]$. It is natural to wonder whether $p[Bel]$ exhibits a similar behaviour with respect to the convex closure operator (see \cite{cuzzolin2021springer}, Chapter 4). Indeed, although the situation is a little more complex in this second case, $p[Bel]$ turns also out to be related to $Cl(.)$ in a rather elegant way. 

Let us introduce the notation $\beta[Bel_i] = N_i/D_i$.
\begin{theorem}\label{the:pdbconvex}
Given two arbitrary belief functions $Bel_1,Bel_2$ defined on the same frame of discernment, the intersection probability of their affine combination $\alpha_1 \; Bel_1$ $+ \; \alpha_2 \; Bel_2$ is, for any $\alpha_1 \in [0,1]$, $\alpha_2 = 1 - \alpha_1$,
\begin{equation}\label{eq:pcl}
\begin{array}{lll}
p[\alpha_1 \; Bel_1 + \alpha_2 \; Bel_2] & = & \displaystyle
\widehat{\alpha_1 D_1} \big (\alpha_1 p[Bel_1] + \alpha_2 T[Bel_1,Bel_2] \big)
\\ \\
& & \displaystyle + \widehat{\alpha_2
D_2} \big (\alpha_1 T[Bel_1,Bel_2]) + \alpha_2 p[Bel_2] \big),
\end{array}
\end{equation}
where $\widehat{\alpha_i D_i} = \frac{\alpha_i D_i}{\alpha_1 D_1+\alpha_2 D_2}$, $T[Bel_1,Bel_2]$ is the probability with values
\begin{equation}\label{eq:t}
T[Bel_1,Bel_2](x) \doteq \hat{D}_1 p[Bel_2,Bel_1] + \hat{D}_2 p[Bel_1,Bel_2],
\end{equation}
with $\hat{D}_i \doteq \frac{D_i}{D_1 + D_2}$, and
\begin{equation}\label{eq:p12}
\begin{array}{c}
p[Bel_2,Bel_1](x) \doteq m_2(x) + \beta[Bel_1] \big (Pl_2(x) - m_2(x) \big ), 
\\ \\
p[Bel_1,Bel_2](x) \doteq m_1(x) + \beta[Bel_2] \big ( Pl_1(x) - m_1(x) \big).
\end{array}
\end{equation}
\end{theorem}

\begin{figure}[ht!]
\begin{center}
\includegraphics[width=0.7 \textwidth]{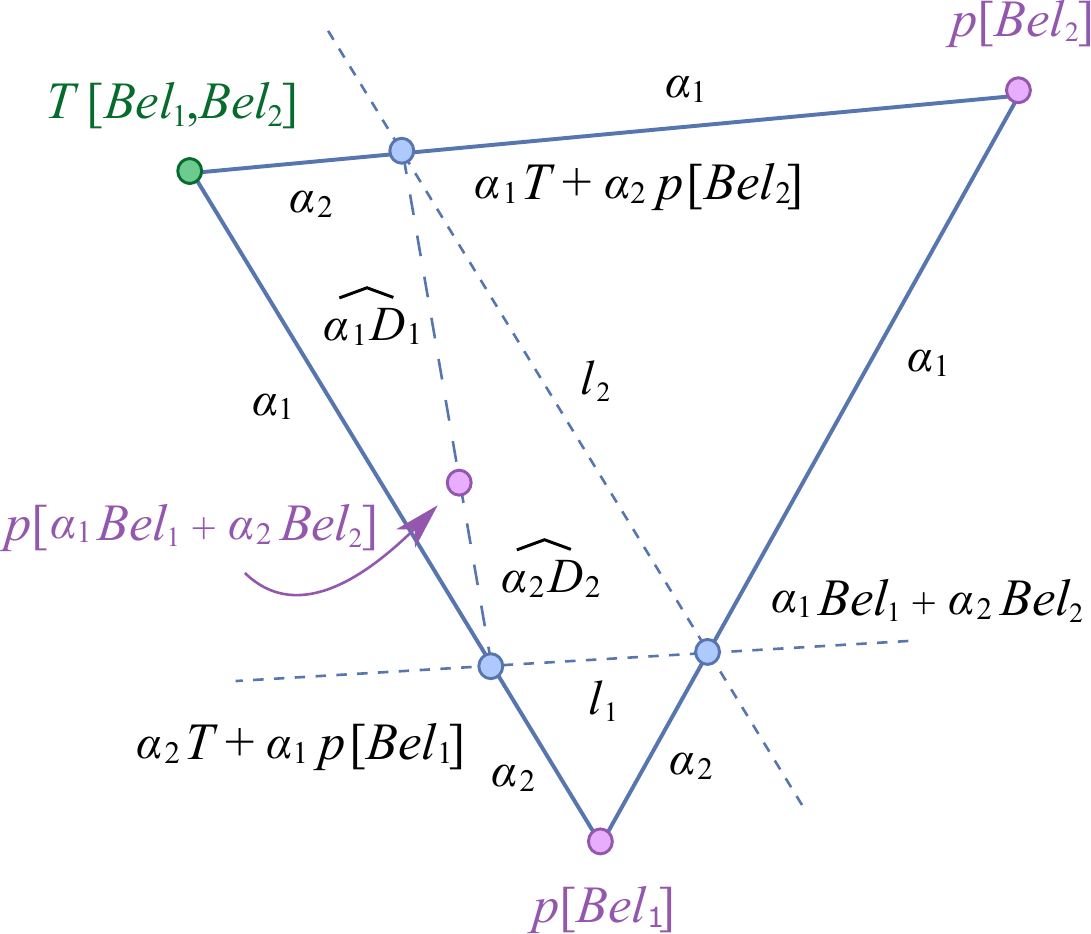}
\end{center}
\caption{Behaviour of the intersection probability $p[Bel]$ under affine combination. $\alpha_2 T + \alpha_1 p[Bel_1]$ and $\alpha_1 T + \alpha_2 p[Bel_2]$ lie in inverted locations on the segments joining $T[Bel_1,Bel_2]$ and $p[Bel_1]$, $p[Bel_2]$, respectively: $\alpha_i p[Bel_i] + \alpha_j T$ is the intersection of the line $Cl(T,p[Bel_i])$ with the line parallel
to $Cl(T,p[Bel_j])$ passing through $\alpha_1 p[Bel_1] + \alpha_2 p[Bel_2]$. 
The quantity $p[\alpha_1 Bel_1 + \alpha_2 Bel_2]$  is, finally, the
point of the segment joining them with convex coordinate $\widehat{\alpha_i D_i}$. \label{fig:T}}
\end{figure} 

Geometrically, $p[\alpha_1 \; Bel_1 + \alpha_2 \; Bel_2]$ can be constructed as in Fig. \ref{fig:T} as a point of the simplex $Cl(T[Bel_1,Bel_2],p[Bel_1],p[Bel_2])$. The point 
\[
\alpha_1 T[Bel_1,Bel_2] + \alpha_2 p[Bel_2] 
\]
is the intersection of the segment $Cl(T,p[Bel_2])$ with the line $l_2$ passing through $\alpha_1 p[Bel_1] + \alpha_2 p[Bel_2]$ and parallel to $Cl(T,p[Bel_1])$.
Dually, the point 
\[
\alpha_2 T[Bel_1,Bel_2] + \alpha_1 p[Bel_1] 
\]
is the intersection of the segment $Cl(T,p[Bel_1])$ with the line $l_1$ passing through $\alpha_1 p[Bel_1] + \alpha_2 p[Bel_2]$ and parallel to $Cl(T,p[Bel_2])$. 
Finally, 
$p[\alpha_1 Bel_1 + \alpha_2 Bel_2]$ is the point of the segment 
\[
Cl(\alpha_1 T + \alpha_2 p[Bel_2],\alpha_2 T + \alpha_1 p[Bel_1])
\] 
with convex coordinate $\widehat{\alpha_1 D_1}$ (or, equivalently, $\widehat{\alpha_2 D_2}$).

\subsubsection{Location of $T[Bel_1,Bel_2]$ in the binary case}

As an example, let us consider the location of $T[Bel_1,Bel_2]$ in the binary belief space $\mathcal{B}_2$ (Fig. \ref{fig:T2}), where
\[
\beta[Bel_1] = \beta[Bel_2] = \frac{m_i(\Theta)}{2 m_i(\Theta)} = \frac{1}{2}
\quad
\forall Bel_1,Bel_2 \in\mathcal{B}_2,
\]
and $p[Bel]$ always commutes with the convex closure operator. 
Accordingly,
\[
\begin{array}{lll}
\displaystyle 
T[Bel_1,Bel_2](x) & = & \displaystyle
\frac{m_1(\Theta)}{m_1(\Theta) + m_2(\Theta)} \left [ m_2(x) + \frac{m_2(\Theta)}{2} \right ] 
\\ \\
& & \displaystyle
+ \frac{m_2(\Theta)}{m_1(\Theta) + m_2(\Theta)} \left [ m_1(x) + \frac{m_1(\Theta)}{2} \right]
\\ \\
& = & \displaystyle 
\frac{m_1(\Theta)}{m_1(\Theta) + m_2(\Theta)} p[Bel_2] + \frac{m_2(\Theta)}{m_1(\Theta) + m_2(\Theta)} p[Bel_1].
\end{array}
\]

\begin{figure}[ht!]
\begin{center}
\includegraphics[width=0.7\textwidth]{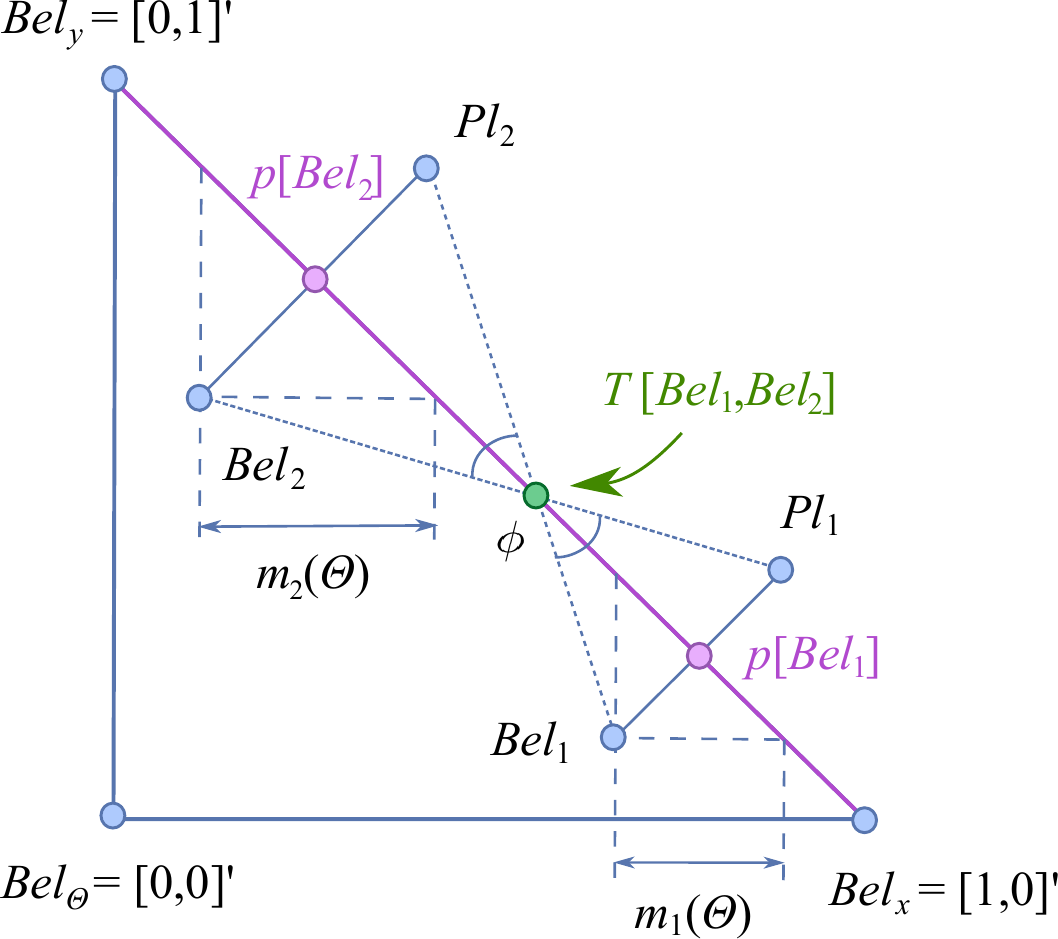}
\end{center}
\caption{Location of the probability function $T[Bel_1,Bel_2]$ in the binary belief space $\mathcal{B}_2$. \label{fig:T2}}
\end{figure} 

Looking at Figure \ref{fig:T2}, simple trigonometric considerations show that the segment defined by $Cl(p[Bel_i],T[Bel_1,Bel_2])$ has length
$m_i(\Theta)/(\sqrt{2}\tan{\phi})$, where $\phi$ is the angle between the segments $Cl(Bel_i,T)$ and $Cl(p[Bel_i],T)$. 
$T[Bel_1,Bel_2]$ is then the unique point of $\mathcal{P}$ such that the angles $\widehat{Bel_1 T p[Bel_1]}$ and $\widehat{Bel_2 T p[Bel_2]}$ coincide, i.e., $T$ is the intersection of $\mathcal{P}$ with the line passing through $Bel_i$ and the reflection of $Bel_j$ through $\mathcal{P}$. 
As this reflection (in $\mathcal{B}_2$) is nothing but $Pl_j$,
\[
T[Bel_1,Bel_2] = Cl(Bel_1,Pl_2) \cap \mathcal{P} = Cl(Bel_2, Pl_1) \cap \mathcal{P}.
\] 

\subsection{Convex closure} \label{sec:operators-closure}

Although the intersection probability does not commute with affine combination (Theorem \ref{the:pdbconvex}), $p[Bel]$ can still be assimilated into the
orthogonal projection and the pignistic function. Theorem \ref{the:pdbcommute} states the conditions under which $p[Bel]$ and convex closure ($Cl$) commute.

\begin{theorem} \label{the:pdbcommute}
The intersection probability and convex closure commute iff
\[
T[Bel_1,Bel_2] = \hat{D}_1 p[Bel_2] + \hat{D}_2 p[Bel_1]
\]
or, equivalently, either $\beta[Bel_1] = \beta[Bel_2]$ or $R[Bel_1] = R[Bel_2]$.
\end{theorem} 

Geometrically, only when the two lines $l_1,l_2$ in Fig. \ref{fig:T} are parallel
to the affine space $a(p[Bel_1],p[Bel_2])$ (i.e., $T[Bel_1,Bel_2] \in Cl(p[Bel_1],p[Bel_2])$; compare the above) does the desired quantity $p[\alpha_1 Bel_1 +
\alpha_2 Bel_2]$ belong to the line segment $Cl(p[Bel_1],p[Bel_2])$ (i.e., it is also a convex combination of $p[Bel_1]$ and $p[Bel_2]$).

Theorem \ref{the:pdbcommute} reflects the two complementary interpretations of $p[Bel]$ we gave in terms of $\beta[Bel]$ and $R[Bel]$:
\[
\begin{array}{l}
p[Bel] = m(x) + (1 - k_{Bel}) R[Bel](x),  
\\ \\
p[Bel] = m(x) + \beta[Bel] (Pl(x) - m(x)).
\end{array}
\]
If $\beta[Bel_1] = \beta[Bel_2]$, both belief functions assign to each singleton the same share of their non-Bayesian contribution.
If $R[Bel_1] = R[Bel_2]$, the non-Bayesian mass $1 - k_{Bel}$ is distributed in the same way to the elements of $\Theta$. 

A sufficient condition for the commutativity of $p[.]$ and $Cl(.)$ can be obtained via
the following decomposition of $\beta[Bel]$:
\begin{equation} \label{eq:beta-sigma}
\begin{array}{l}
\displaystyle \beta[Bel] = \frac{\sum_{|B|>1} m(B)}{\sum_{|B|>1} m(B) |B|} = \frac{\sum_{k = 2}^n \sum_{|B| = k}
m(B)}{ \sum_{k = 2}^n k \cdot \sum_{|B| = k} m(B) } = \frac{\sigma_2 + \cdots + \sigma_n}{ 2 \sigma_2 + \cdots + n \sigma_n},
\end{array}
\end{equation}
where $\sigma_k \doteq \sum_{|B|=k} m(B)$.

\begin{theorem} \label{the:pdb-sufficient-convex}
If the ratio between the total masses of focal elements of different cardinality is the same for all the belief functions involved, namely
\begin{equation} \label{eq:pdb-sufficient-convex}
\begin{array}{cccc}
\displaystyle \frac{\sigma_1^l}{\sigma_1^m} = \frac{\sigma_2^l}{\sigma_2^m} & \quad \forall l,m \geq 2 & \text{s.t.} & \sigma_1^m, \sigma_2^m \neq 0,
\end{array}
\end{equation}
then the intersection probability (considered as an operator mapping belief functions to probabilities) and convex combination commute.
\end{theorem}

\section{Conclusions} \label{sec:decision}

The intersection probability possesses a simple credal interpretation in the probability simplex (Section \ref{sec:credal}), as it can be linked to a pair of credal sets, thus in turn extending the classical interpretation of the pignistic transformation as the barycentre of the polygon of consistent probabilities. 

Just like $\mathcal{P}[Bel]$ is the credal set associated with a belief function $Bel$, the upper and lower simplices geometrically embody the probability interval associated with $Bel$:
\[
\mathcal{P}[(Bel,Pl)] = \Big \{ p \in \mathcal{P} : Bel(x) \leq p(x) \leq Pl(x), \forall x \in \Theta \Big \}.
\]
The intersection probability turns out to be the focus to the pair $\{ T^1[Bel], T^{n-1}[Bel] \}$ of lower and upper simplices:
\begin{equation} \label{eq:star-credal}
f(T^1[Bel], T^{n-1}[Bel]) = p[Bel].
\end{equation}
Its coordinates as foci encode major features of the underlying belief function, in particular the fraction $\beta$ of the related probability interval which yields the intersection probability. 

This credal interpretation 
hints at the possible formulation of a decision making framework for probability intervals analogous to Smets's \emph{transferable belief model} (TBM) \cite{smets93belief}.
We can think of the TBM as a pair $\{ \mathcal{P}[Bel], BetP[Bel] \}$ formed by a credal set linked to each belief function $Bel$ (in this case the polytope of consistent probabilities) and a probability transformation (the pignistic function).
As the barycentre of a simplex is a special case of a focus, the pignistic transformation is just another probability transformation induced by the focus of two simplices.

The results in this paper therefore suggest a similar framework
\[
\Big \{ \big \{ T^{1}[Bel], T^{n-1}[Bel] \big \}, p[Bel] \Big \},
\]
in which interval constraints on probability distributions on $\mathcal{P}$ are represented by a similar pair, formed by the above pair of simplices and by the probability transformation identified by their focus. Decisions are then made based on the appropriate focus probability, i.e., the intersection probability.

In the TBM \cite{smets93belief}, disjunctive/conjunctive combination rules are applied to belief functions to update or revise our state of belief according to new evidence. The formulation of a similar alternative frameworks for interval probability systems would then require us to design specific evidence elicitation/revision operators for them.
This elicits a number of questions: how can we design such operator(s)? How are they related to combination rules for belief functions in the case of probability intervals induced by belief functions?

We will further explore the betting interpretation of the intersection probability in the near future.



\section*{Appendix}

\subsection*{Proof of Theorem \ref{the:simplices}}

\begin{lemma} \label{lem:affine-tx}
The points $\{ t_x^1[Bel], x \in \Theta \}$ are affinely independent.
\end{lemma}
\begin{proof} 
Let us suppose, contrary to the thesis, that there exists an affine decomposition of one of the points, say $t_x[Bel]$, in terms of the others: 
\[
t_x^1[Bel] = \sum_{z \neq x} \alpha_z t_z^1[Bel], \quad \alpha_z \geq 0 \; \forall z \neq x, \; \sum_{z \neq x} \alpha_z = 1. 
\]
But then we would have, by the definition of $t_z^1[Bel]$,
\[
\begin{array}{lll}
t_x^1[Bel] & = & \displaystyle \sum_{z \neq x} \alpha_z t_z^1[Bel] 
\\
& = & 
\displaystyle
\sum_{z \neq x} \alpha_z \left ( \sum_{y \neq z} m(y) Bel_y \right ) + \sum_{z \neq x} \alpha_z \big ( m(z) + 1 - k_{Bel} \big ) Bel_z 
\\ \\
& = & 
\displaystyle 
m(x) Bel_x \sum_{z \neq x} \alpha_z + \sum_{z \neq x} Bel_z m(z) (1 - \alpha_z) 
\end{array}
\]

\[
\begin{array}{lll}
& & 
+ \displaystyle 
\sum_{z \neq x} \alpha_z m(z) Bel_z + (1 - k_{Bel}) \sum_{z \neq x} \alpha_z Bel_z 
\\ \\
& = & 
\displaystyle 
\sum_{z \neq x} m(z) Bel_z + m(x) Bel_x + (1 - k_{Bel}) \sum_{z \neq x} \alpha_z Bel_z.
\end{array}
\]
The latter is equal to (\ref{eq:t1-vertices})
\[
t_x^1[Bel] = \sum_{z \neq x} m(z) Bel_z + (m(x) + 1 - k_{Bel}) Bel_x
\]
if and only if
$\sum_{z \neq x} \alpha_z Bel_z = Bel_x.$
But this is impossible, as the categorical probabilities $Bel_x$ are trivially affinely independent.
\end{proof}

\subsection*{Proof of Theorem \ref{the:focus-intersection}}

Suppose that $p$ is a special focus of $S$ and $T$, namely $\exists \alpha \in \mathbb{R}$ such that
\[
p = \alpha s_i + (1 - \alpha) t_{\rho(i)} \quad \forall i = 1, \ldots, n,
\]
for some permutation $\rho$ of $\{1, \ldots, n\}$.
Then, necessarily,
\[
t_{\rho(i)} = \frac{1}{1-\alpha} [ p - \alpha s_i ]  \quad \forall i = 1, \ldots, n. 
\]
If $p$ has coordinates $\{\alpha_i, i = 1, \ldots, n\}$ in $T$, $p = \sum_{i = 1}^n \alpha_i t_{\rho(i)}$, it follows that
\[
\begin{array}{lll}
p & = & \displaystyle \sum_{i = 1}^n \alpha_i t_{\rho(i)} = \frac{1}{1-\alpha} \sum_i \alpha_i \big [ p - \alpha s_i \big ]
\\ \\
 & = & \displaystyle \frac{1}{1 - \alpha} \left [ p \sum_i \alpha_i - \alpha \sum_i \alpha_i s_i \right ] = \frac{1}{1-\alpha} \left [ p - \alpha \sum_i \alpha_i s_i \right ].
\end{array}
\]
The latter implies that $p = \sum_i \alpha_i s_i$, i.e., $p$ has the same simplicial coordinates in $S$ and $T$.

\subsection*{Proof of Theorem \ref{the:pdb-focus}}

The common simplicial coordinates of $p[Bel]$ in $T^1[Bel]$ and $T^{n-1}[Bel]$ turn out to be the values of the relative uncertainty function (\ref{eq:relative-uncertainty}) for $Bel$,
\begin{equation} \label{eq:r-belief}
R[Bel](x) = \frac{Pl(x) - m(x)}{k_{Pl} - k_{Bel}}.
\end{equation}
Recalling the expression (\ref{eq:t1-vertices}) for the vertices of $T^1[Bel]$, the point of the simplex $T^1[Bel]$ with coordinates (\ref{eq:r-belief}) is
\[
\begin{array}{lll}
& &
\displaystyle 
\sum_x R[Bel](x) t_x^1[Bel] 
\\
& = & 
\displaystyle 
\sum_x R[Bel](x) \left [ \sum_{y \neq x} m(y) Bel_y + \bigg ( 1 - \sum_{y \neq x} m(y) \bigg ) Bel_x \right ] 
\\ \\
& = & 
\displaystyle 
\sum_x R[Bel](x) \left [ \sum_{y \in\Theta } m(y) Bel_y + (1 - k_{Bel}) Bel_x \right ] 
\\ \\
& = & 
\displaystyle 
\sum_x Bel_x \left [ (1 - k_{Bel}) R[Bel](x) + m(x) \sum_y R[Bel](y) \right ] 
\\ \\
& = & 
\displaystyle 
\sum_x Bel_x \Big [ (1 - k_{Bel}) R[Bel](x) + m(x) \Big ],
\end{array}
\]
as $R[Bel]$ is a probability ($\sum_y R[Bel](y) = 1$). By Equation (\ref{eq:form2}), the above quantity coincides with $p[Bel]$. 

The point of $T^{n-1}[Bel]$ with the same coordinates, $\{ R[Bel](x), x \in \Theta \}$, is
\[
\begin{array}{lll}
& &
\displaystyle 
\sum_x R[Bel](x) t_x^{n-1}[Bel] 
\\
& = & 
\displaystyle 
\sum_x R[Bel](x) \left [ \sum_{y \neq x} Pl(y) Bel_y + \bigg ( 1 - \sum_{y \neq x} Pl(y) \bigg ) Bel_x \right ] 
\\ \\
& = & 
\displaystyle 
\sum_x R[Bel](x) \left [ \sum_{y \in\Theta } Pl(y) Bel_y + (1 - k_{Pl}) Bel_x \right ] 
\\ \\
& = & 
\displaystyle \sum_x Bel_x \left [ (1 - k_{Pl}) R[Bel](x) + Pl(x) \sum_y R[Bel](y) \right ] 
\\ \\
& = & 
\displaystyle 
\sum_x Bel_x \Big [ (1 - k_{Pl}) R[Bel](x) + Pl(x) \Big ] 
\\ \\
& = &  
\displaystyle
\sum_x Bel_x \bigg [ Pl(x) \frac{1 - k_{Bel}}{k_{Pl} - k_{Bel}} - m(x) \frac{1 - k_{Bel}}{k_{Pl} - k_{Bel}} \bigg ],
\end{array}
\]
which is equal to $p[Bel]$ by (\ref{eq:r-belief}). 

\subsection*{Proof of Theorem \ref{the:pdb-coordinate}}

Again, we need to impose the condition (\ref{eq:condition}) on the pair $\{ T^1[Bel], T^{n-1}[Bel] \}$, namely
\[
p[Bel] = t_x^1[Bel] + \alpha (t_x^{n-1}[Bel] - t_x^1[Bel]) = (1 - \alpha) t_x^1[Bel] + \alpha t_x^{n-1}[Bel]
\]
for all the elements $x \in \Theta$ of the frame, $\alpha$ being some constant real number. This is equivalent to (after substituting the expressions (\ref{eq:t1-vertices}), (\ref{eq:tn-1-vertices}) for $t_x^1[Bel]$ and $t_x^{n-1}[Bel]$)
\[
\begin{array}{lll}
& &
\displaystyle \sum_{x \in\Theta} Bel_x \Big [ m(x) + \beta[Bel] (Pl(x) - m(x)) \Big ] 
\\ 
& = & \displaystyle 
(1-\alpha) \left [ \sum_{y \in\Theta} m(y) Bel_y + (1 - k_{Bel}) Bel_x \right ]
\\  \\
& & \displaystyle
+ \alpha \left [ \sum_{y \in\Theta} Pl(y) Bel_y + ( 1 - k_{Pl}) Bel_x \right ] 
\\ \\
& = & \displaystyle 
Bel_x \Big [ (1-\alpha) (1 - k_{Bel}) + (1-\alpha) m(x) + \alpha Pl(x) + \alpha (1 - k_{Pl}) \Big ] 
\\ \\
& & \displaystyle 
+ \sum_{y \neq x} Bel_y \Big [ (1-\alpha) m(y) + \alpha Pl(y) \Big ] 
\\ \\
& = & \displaystyle 
Bel_x \Big \{ (1 - k_{Bel}) + m(x) + \alpha \big [ Pl(x) + (1 - k_{Pl}) - m(x) - (1 - k_{Bel}) \big ] \Big \} 
\\ \\
& & \displaystyle  + \sum_{y \neq x} Bel_y \Big [ m(y) + \alpha \big ( Pl(y) - m(y)  \big ) \Big ].
\end{array}
\]
If we set $\alpha = \beta[Bel] = (1 - k_{Bel}) / (k_{Pl} - k_{Bel})$, we get the following for the coefficient of $Bel_x$ in the above expression (i.e., the probability value of $x$): 
\[
\begin{array}{lll}
& & \displaystyle
\frac{1 - k_{Bel}}{k_{Pl} - k_{Bel}} \big [ Pl(x) + (1 - k_{Pl}) - m(x) - (1 - k_{Bel}) \big ] + (1 - k_{Bel}) + m(x)
\\ \\
& = & \beta[Bel] [ Pl(x) - m(x) ] + (1 - k_{Bel}) + m(x) - (1 - k_{Bel}) = p[Bel](x). 
\end{array}
\]
On the other hand,
\[
m(y) + \alpha (Pl(y) - m(y) ) = m(y) + \beta[Bel] (Pl(y) - m(y) ) = p[Bel](y) 
\]
for all $y \neq x$, no matter what the choice of $x$.

\subsection*{Proof of Theorem \ref{the:pdbconvex}}

By definition, the quantity $p[\alpha_1 Bel_1 + \alpha_2 Bel_2](x)$ can be written as
\begin{equation}\label{eq:palpha}
m_{\alpha_1 Bel_1 + \alpha_2 Bel_2} (x) + \beta [\alpha_1 Bel_1 + \alpha_2 Bel_2]
\sum_{A\supsetneq \{x\}} m_{\alpha_1 Bel_1 + \alpha_2 Bel_2}(A),
\end{equation}
where
\[
\begin{array}{lll}
& & \displaystyle
\beta[\alpha_1 Bel_1 + \alpha_2 Bel_2]
\\ \\
 & = & \displaystyle 
\frac{\displaystyle \sum_{|A|>1} m_{\alpha_1 Bel_1 + \alpha_2 Bel_2}(A)}
{\displaystyle \sum_{|A|>1} m_{\alpha_1 Bel_1 + \alpha_2 Bel_2}(A) |A|} 
= 
\frac{\displaystyle \alpha_1 \sum_{|A|>1}m_1(A) + \alpha_2 \sum_{|A|>1}m_2(A)}{\displaystyle \alpha_1 \sum_{|A|>1} m_1(A) |A| + \alpha_2 \sum_{|A|>1} m_2(A) |A|} 
\\ \\ 
& = & \displaystyle \frac{\alpha_1 N_1 + \alpha_2 N_2}{\alpha_1 D_1 + \alpha_2 D_2}
= \frac{\alpha_1 D_1 \beta[Bel_1] + \alpha_2 D_2 \beta[Bel_2]}{\alpha_1 D_1 + \alpha_2 D_2} 
\\ \\
& = & \widehat{\alpha_1 D_1} \beta[Bel_1] + \widehat{\alpha_2 D_2} \beta[Bel_2],
\end{array}
\]
once we introduce the notation $\beta[Bel_i] = N_i/D_i$.  

Substituting this decomposition for $\beta[\alpha_1 Bel_1 + \alpha_2 Bel_2]$ into (\ref{eq:palpha}) yields
\begin{equation} \label{eq:proof-T}
\begin{array}{lll}
& & \displaystyle \alpha_1 m_1(x) + \alpha_2 m_2(x)
\\
& & \displaystyle
+ \Big ( \widehat{\alpha_1 D_1} \beta[Bel_1] + \widehat{\alpha_2 D_2} \beta[Bel_2] \Big ) \left ( \alpha_1 \sum_{A\supsetneq \{x\}} m_1(A) + \alpha_2 \sum_{A\supsetneq \{x\}} m_2(A) \right ) 
\\ \\
& = & \displaystyle \frac{(\alpha_1 D_1+\alpha_2 D_2)(\alpha_1 m_1(x) + \alpha_2 m_2(x))}{\alpha_1 D_1+\alpha_2 D_2} 
\\ \\
& & \displaystyle + \frac{\alpha_1 D_1 \beta[Bel_1] + \alpha_2 D_2 \beta[Bel_2]}{\alpha_1 D_1+\alpha_2 D_2} \left ( \alpha_1 \sum_{A\supsetneq \{x\}} m_1(A) + \alpha_2 \sum_{A\supsetneq \{x\}} m_2(A) \right),
\end{array}
\end{equation}
which can be reduced further to
\[
\begin{array}{lll}
& & \displaystyle \frac{\alpha_1 D_1}{\alpha_1 D_1+\alpha_2 D_2} \Big ( \alpha_1 m_1(x) + \alpha_2 m_2(x) \Big ) 
\\
& & + \displaystyle 
\frac{\alpha_1 D_1}{\alpha_1 D_1+\alpha_2 D_2} \beta[Bel_1] \left ( \alpha_1 \sum_{A\supsetneq \{x\}} m_1(A) + \alpha_2 \sum_{A\supsetneq \{x\}} m_2(A) \right ) 
\\
& & \displaystyle
+ \frac{\alpha_2 D_2}{\alpha_1 D_1+\alpha_2 D_2} \Big ( \alpha_1 m_1(x) + \alpha_2 m_2(x) 
\Big )
\\
& & \displaystyle
+ \frac{\alpha_2 D_2}{\alpha_1 D_1+\alpha_2 D_2} \beta[Bel_2] \left ( \alpha_1 \sum_{A\supsetneq \{x\}} m_1(A) + \alpha_2 \sum_{A\supsetneq \{x\}} m_2(A) \right) 
\\ \\
& = & \displaystyle \frac{\alpha_1 D_1}{\alpha_1 D_1+\alpha_2 D_2} \alpha_1 \left ( m_1(x) + \beta[Bel_1] \sum_{A\supsetneq \{x\}} m_1(A) \right ) 
\end{array}
\]

\begin{equation}
\begin{array}{lll}
& & \displaystyle
+ \frac{\alpha_1 D_1}{\alpha_1 D_1+\alpha_2 D_2} \alpha_2 \left ( m_2(x) + \beta[Bel_1] \sum_{A\supsetneq \{x\}} m_2(A) \right ) 
\\ \\
& & + \displaystyle \frac{\alpha_2 D_2}{\alpha_1 D_1 + \alpha_2 D_2} \alpha_1 \left (
m_1(x) + \beta[Bel_2] \sum_{A\supsetneq \{x\}} m_1(A) \right ) 
\\ \\
& & + \displaystyle \frac{\alpha_2 D_2}{\alpha_1 D_1 + \alpha_2 D_2} \alpha_2 \left ( m_2(x) + \beta[Bel_2] \sum_{A\supsetneq \{x\}} m_2(A) \right ) 
\\ \\
& = &  \displaystyle
\frac{\alpha^2_1 D_1}{\alpha_1 D_1+\alpha_2 D_2} p[Bel_1] + \frac{\alpha^2_2 D_2}{\alpha_1 D_1+\alpha_2 D_2} p[Bel_2] 
\\ \\
& & \displaystyle
+ \frac{\alpha_1 \alpha_2}{\alpha_1 D_1 + \alpha_2 D_2} \Big ( D_1 p[Bel_2,Bel_1](x) + D_2 p[Bel_1,Bel_2](x) \Big ),
\end{array}
\end{equation}
after recalling (\ref{eq:p12}). 

We can notice further that the function
\[
F(x) \doteq D_1 p[Bel_2,Bel_1](x) + D_2 p[Bel_1,Bel_2](x)
\]
is such that
\[
\begin{array}{ll}
&
\displaystyle
\sum_{x \in \Theta} F(x)
\\
=
&
\displaystyle 
\sum_{x \in \Theta} \big [ D_1 m_2(x) + N_1 (Pl_2- m_2(x)) + D_2 m_1(x) + N_2
(Pl_1 - m_1(x)) \big ] 
\\ 
= 
&
D_1 (1 - N_2) + N_1 D_2 + D_2 (1 - N_1) + N_2 D_1 = D_1 + D_2
\end{array}
\]
(making use of (\ref{eq:beta})). Thus, $T[Bel_1,Bel_2](x) = F(x)/(D_1+D_2)$ is a probability (as $T[Bel_1,Bel_2](x)$ is always non-negative), expressed by (\ref{eq:t}). 
By (\ref{eq:proof-T}), the quantity $p[\alpha_1 Bel_1 + \alpha_2 Bel_2](x)$ can be expressed as
\[
\frac{\alpha_1^2 D_1 p[Bel_1](x) + \alpha_2^2 D_2 p[Bel_2](x) + \alpha_1 \alpha_2 (D_1 + D_2) T[Bel_1,Bel_2](x)}{\alpha_1 D_1+\alpha_2 D_2},
\]
i.e., (\ref{eq:pcl}). 

\subsection*{Proof of Theorem \ref{the:pdbcommute}}

By (\ref{eq:pcl}), we have that
\[
\begin{array}{lll}
& & p[\alpha_1 Bel_1+\alpha_2 Bel_2] - \alpha_1 p[Bel_1] - \alpha_2 p[Bel_2]
\\ \\
& = &  \widehat{\alpha_1 D_1} \alpha_1 p[Bel_1] + \widehat{\alpha_1 D_1} \alpha_2 T + \widehat{\alpha_2 D_2} \alpha_1 T + \widehat{\alpha_2 D_2} \alpha_2 p[Bel_2] 
\\ \\
& & - \alpha_1 p[Bel_1] - \alpha_2 p[Bel_2] 
\\ \\
& = & \alpha_1 p[Bel_1] (\widehat{\alpha_1 D_1} -1) + \widehat{\alpha_1 D_1} \alpha_2 T + \widehat{\alpha_2 D_2} \alpha_1 T + \alpha_2 p[Bel_2] (\widehat{\alpha_2 D_2}-1) 
\\ \\
& = & - \alpha_1 p[Bel_1] \widehat{\alpha_2 D_2} + \widehat{\alpha_1 D_1} \alpha_2 T + \widehat{\alpha_2 D_2} \alpha_1 T - \alpha_2 p[Bel_2] \widehat{\alpha_1 D_1}
\\ \\
& = & \widehat{\alpha_1 D_1} \Big (\alpha_2 T - \alpha_2 p[Bel_2] \Big ) + \widehat{\alpha_2 D_2} \Big (\alpha_1 T - \alpha_1 p[Bel_1] \Big ) 
\\ \\
& = & 
\displaystyle
\frac{\displaystyle \alpha_1 \alpha_2}{\displaystyle\alpha_1 D_1 + \alpha_2 D_2} \Big [D_1 (T-p[Bel_2]) + D_2 (T-p[Bel_1]) \Big ].
\end{array}
\]
This is zero iff
\[
T[Bel_1,Bel_2] (D_1+D_2) = p[Bel_1] D_2 + p[Bel_2] D_1,
\]
which is equivalent to
\[
T[Bel_1,Bel_2] = \hat{D}_1 p[Bel_2] + \hat{D}_2 p[Bel_1],
\]
as $(\alpha_1 \alpha_2) / (\alpha_1 D_1 + \alpha_2 D_2)$ is always non-zero in non-trivial cases. This is equivalent to (after replacing the expressions for $p[Bel]$ (\ref{eq:interpretation2}) and $T[Bel_1,Bel_2]$ (\ref{eq:t}))
\[
D_1 (Pl_2- m_2(x))(\beta[Bel_2] - \beta[Bel_1]) + D_2 (Pl_1 - m_1(x))(\beta[Bel_1] - \beta[Bel_2]) = 0,
\]
which is in turn equivalent to
\[
\begin{array}{c}
(\beta[Bel_2] - \beta[Bel_1]) \Big[ D_1 (Pl_2(x) - m_2(x)) - D_2 (Pl_1(x) - m_1(x)) \Big] = 0.
\end{array}
\]
Obviously, this is true iff $\beta[Bel_1] = \beta[Bel_2]$ \emph{or} the second factor is zero, i.e.,
\[
\begin{array}{lll}
& & 
\displaystyle D_1 D_2 \frac{Pl_2(x) - m_2(x)}{D_2} - D_1 D_2 \frac{Pl_1(x) - m_1(x)}{D_1}
\\ \\
& = & D_1 D_2 (R[Bel_2](x) - R[Bel_1](x)) = 0
\end{array}
\]
for all $x\in\Theta$, i.e., $R[Bel_1] = R[Bel_2]$. 

\subsection*{Proof of Theorem \ref{the:pdb-sufficient-convex}}

By (\ref{eq:beta-sigma}), the equality $\beta[Bel_1] = \beta[Bel_2]$ is equivalent to
\[
(2 \sigma_2^2 + \cdots + n \sigma_2^n) (\sigma_1^2 + \cdots + \sigma_1^n) 
= 
(2 \sigma_1^2 + \cdots + n \sigma_1^n) (\sigma_2^2 + \cdots + \sigma_2^n).
\]
Let us assume that there exists a cardinality $k$ such that
$\sigma_1^k \neq 0 \neq \sigma_2^k$. We can then divide the two sides by $\sigma_1^{k}$ and $\sigma_2^{k}$, obtaining
\[
\begin{array}{lll}
& & \displaystyle 
\left ( 2 \frac{\sigma_2^2}{\sigma_2^{k}} + \cdots + k + \cdots + n \frac{\sigma_2^n}{\sigma_2^{k}} \right) \left ( \frac{\sigma_1^2}{\sigma_1^{k}} + \cdots + 1 + \cdots + \frac{\sigma_1^n}{\sigma_1^{k}} \right) 
\\ \\
& = & \displaystyle \left ( 2 \frac{\sigma_1^2}{\sigma_1^{k}} + \cdots + k + \cdots + n \frac{\sigma_1^n}{\sigma_1^{k_1}} \right ) \left ( \frac{\sigma_2^2}{\sigma_2^{k}} + \cdots + 1 + \cdots + \frac{\sigma_2^n}{\sigma_2^{k}} \right ).
\end{array}
\]
Therefore, if $\sigma_1^j / \sigma_1^k = \sigma_2^j / \sigma_2^k$ $\forall j\neq k$, the condition $\beta[Bel_1] = \beta[Bel_2]$ is satisfied. But this is equivalent to (\ref{eq:pdb-sufficient-convex}).

\bibliographystyle{plain}
\bibliography{arxiv-ijcai-survey}

\end{document}